\documentclass{article} %
\usepackage{iclr2027_conference,times}

\usepackage[utf8]{inputenc} %
\usepackage[T1]{fontenc}    %
\usepackage{hyperref}       %
\usepackage{url}            %
\usepackage{booktabs}       %
\usepackage{nicefrac}       %
\usepackage{microtype}      %
\usepackage{xcolor}         %

\usepackage{algorithmic}
\usepackage{algorithm}

\usepackage{srcltx}
\usepackage{etoolbox}
\usepackage{nameref}
\usepackage{comment}
\usepackage{mathrsfs}
\usepackage{enumerate}
\usepackage{amsfonts}
\usepackage{amsmath}
\usepackage{amssymb}
\usepackage{mathtools}
\usepackage{amsthm}
\usepackage{enumitem}

\usepackage{nicefrac} %
\usepackage{mathtools}
\usepackage{dsfont,mathrsfs}
\usepackage{xargs}
\usepackage{multirow}

\usepackage{cleveref}

\usepackage{autonum}

\usepackage{upgreek}

\usepackage{xfrac}

\newtheorem{assum}{A\hspace{-2pt}}

\crefname{assum}{A\hspace{-2pt}}{A\hspace{-2pt}}
\crefname{assumb}{B\hspace{-2pt}}{B\hspace{-2pt}}
\crefname{assumUGE}{UGE\hspace{-1pt}}{UGE\hspace{-1pt}}
\crefname{assumID}{IND\hspace{-1pt}}{IND\hspace{-1pt}}
\crefname{assumUE}{UE\hspace{-1pt}}{UE\hspace{-1pt}}
\crefname{assumSUP}{M\hspace{-1pt}}{M\hspace{-1pt}}
\theoremstyle{plain}
\newtheorem{theorem}{Theorem}[section]

\newtheorem{lemma}[theorem]{Lemma}
\newtheorem{claim}[theorem]{Claim}
\newtheorem{corollary}[theorem]{Corollary}
\theoremstyle{definition}

\theoremstyle{remark}

\usepackage{graphicx}
\usepackage{subfigure}

\renewcommand{\S}{\mathcal{S}}
\newcommand{\A}{\mathcal{A}}
\def\PMDP{P}

\def\covfeat{\Sigma^\varphi}
\def\FLSA{\textsf{FedLSA}}

\def\minlambda{\nu}

\newcommand{\perlsa}{\hyperref[{alg:personalized_lsa}]{\textsf{PF-LSA}}}

\newcommandx{\supconstepc}[1][1=]{\supnorm{\funnoisew_c}}
\newcommandx\conststab[1][1=p]{\varkappa_{#1}}

\newcommand{\PE}{\mathbb{E}}
\newcommand{\PP}{\mathbb{P}}
\newcommandx{\genericb}[1][1=]{b_{#1}}
\newcommandx{\Constros}[1][1=]{\operatorname{C}_{\operatorname{Ros},#1}}
\newcommandx{\Constburk}[1][1=]{\operatorname{C}_{\operatorname{Burk}}}
\newcommandx{\driftW}[1][1=]{W_{#1}}

\newcommandx{\metricd}[1][1=]{\mathsf{d}_{#1}}

\newcommandx\invmeasure[1][1=]{\Pi_{#1}}
\newcommandx{\PPjoint}[1][1=]{\PP^{\MKjoint[#1]}}
\newcommandx{\PEjoint}[1][1=]{\PE^{\MKjoint[#1]}}
\newcommandx{\PEMID}[1][1=\alpha]{\PE^{\MK[#1]}}
\newcommandx{\PPMID}[1][1=\alpha]{\PP^{\MK[#1]}}

\newcommand{\supnorm}[1]{\norm{ #1 }[\infty]}
\newcommandx{\MKjoint}[1][1=]{\bar{\operatorname{P}}_{#1}}
\newcommandx\costw[1][1=]{\mathsf{c}_{#1}}

\newcommandx\Intergrdist[1][1=]{\mathbb{M}_{1}(#1)}

\newcommandx{\mmarkov}[1][1=0]{m^{(\Markov)}_{#1}}

\def\Zset{\mathsf{Z}}

\def\rset{\mathbb{R}}

\def\nset{\ensuremath{\mathbb{N}}}
\def\nsets{\ensuremath{\mathbb{N}^*}}

\newcommand{\bConst}[1]{\operatorname{C}_{{\bf #1}}}

\def\avgnoise{\bar{\sigma}_{\varepsilon}}

\newcommandx\sequence[4][2=,3=,4=]
{\ifthenelse{\equal{#3}{}}{\ensuremath{\{ #1_{#2 #4}\}}}{\ensuremath{\{ #1_{#2 #4} \}_{#2 \in #3}}}}

\newcommandx\sequenceD[2][2=]
{\ifthenelse{\equal{#2}{}}{\ensuremath{\{ #1\}}}{\ensuremath{\{ #1\!~:\!~#2  \}}}}
\newcommandx\sequenceDouble[4][3=,4=]
{\ifthenelse{\equal{#3}{}}{\ensuremath{\{ (#1_{#3},#2_{#3})\}}}{\ensuremath{\{ (#1_{#3},#2_{#3})\}_{#3 \in #4}}}}

\newcommandx{\sequencen}[2][2=n\in\nset]{\ensuremath{\{ #1, \eqsp #2 \}}}
\newcommandx\sequencens[2][2=n]
{\ensuremath{\{ #1_{#2} \!~:\!~#2\in\nsets\}}}
\newcommandx\sequencet[4]
{\ensuremath{\{ #1{#2_{#3}} \, : \, \eqsp #3 \in #4 \}}}
\def\PE{\mathbb{E}}

\def\ProdB{\Gamma}

\newcommandx{\ProdBatc}[2]{\ProdB_{#1}^{(#2)}}
\newcommandx{\avgProdB}[2]{\bar{\Gamma}_{#1}^{(#2)}}
\newcommandx{\avgProdDetB}[2]{\bar{\Gamma}^{(#2)}_{#1}}
\newcommandx{\RavgProdDetB}[2]{\bar{\Theta}^{(#2)}_{#1}}
\newcommandx{\avgbias}[1]{\bar{\rho}_{#1}}
\newcommandx{\avgflbias}[2]{\bar{\tau}_{#1,#2}}
\newcommandx{\bbias}[2]{\rho_{#2}}
\newcommandx{\vbias}[2]{\tilde{v}_{#1,#2}}
\newcommandx{\avgfl}[2]{\bar{\varphi}_{#1,#2}}

\newcommandx{\PVar}[1][1=]{\ensuremath{\operatorname{Var}_{#1}}}
\newcommandx{\CovfuncA}[1][1=c]{\Sigma^{#1}_{\zmfuncAw}}
\newcommandx{\Covfuncb}[1][1=c]{V^{#1}_{\zmfuncbw}}

\def\noisecov{\Sigma^\varepsilon}

\newcommandx{\MK}[1][1=\alpha]{\mathrm{P}_{#1}}
\newcommandx\MKK[1][1=\alpha]{\mathrm{K}_{#1}}

\def\filtr{\mathcal{F}}

\newcommandx{\PEtilde}[1][1=]{\PE^{\mathrm{K}_{#1}}}
\newcommandx{\PPtilde}[1][1=]{\PP^{\mathrm{K}_{#1}}}

\newcommandx{\norm}[2][2=]{\Vert#1 \Vert_{{#2}}}
\newcommandx{\bnorm}[2][2=]{\Big\Vert#1 \Big\Vert_{{#2}}}
\newcommandx{\normLigne}[2][2=]{\Vert#1 \Vert_{{#2}}}
\newcommandx{\normLine}[2][2=]{\Vert#1 \Vert_{{#2}}}
\newcommandx{\normop}[2][2=]{\Vert{#1}\Vert_{{#2}}}
\newcommandx{\normopLigne}[2][2=]{\Vert{#1}\Vert_{{#2}}}
\newcommandx{\normopLine}[2][2=]{\Vert{#1}\Vert_{{#2}}}
\newcommandx{\osc}[2][1=]{\mathrm{osc}_{#1}(#2)}

\newcommandx{\normlip}[2][2=\operatorname{Lip}]{\Vert#1 \Vert_{{#2}}}
\newcommand{\lip}{\operatorname{L}}
\newcommandx{\lipspace}[1]{\lip_{#1}}

\newcommandx{\CPP}[3][1=]
{\ifthenelse{\equal{#1}{}}{{\mathbb P}\left(\left. #2 \, \right| #3 \right)}{{\mathbb P}_{#1}\left(\left. #2 \, \right | #3 \right)}}
\newcommandx{\CPPtilde}[3][1=]
{\ifthenelse{\equal{#1}{}}{{\tilde{\mathbb P}}\left(\left. #2 \, \right| #3 \right)}{{\tilde{\mathbb P}}_{#1}\left(\left. #2 \, \right | #3 \right)}}

\def\etal{et al.}

\def\iid{i.i.d.}

\newcommandx{\as}[1][1=\PP]{\ensuremath{#1\, -\mathrm{a.s.}}}

\newcommand{\eqsp}{\;}

\newcommand{\Id}{\mathrm{I}}

\newcommandx{\boundmetric}[1][1=]{\kappa_{\MKK[#1]}}

\newcommandx{\Nnorm}[2][1=V]{[ #2]_{#1}}
\newcommandx{\lipnorm}[2][1=g]{[ #1]_{#2}}

\newcommandx{\CPE}[3][1=]{{\mathbb E}^{#3}_{#1}\left[#2\right]}
\newcommandx{\CPEext}[3][1=]{\tilde{\mathbb E}^{#3}_{#1}\left[#2\right]}
\newcommandx{\CPEtilde}[3][1=]{{\tilde{\mathbb E}}^{#3}_{#1}\left[#2\right]}
\newcommandx{\CPEs}[3][1=]{{\mathbb E}^{#3}_{#1}[#2]}

\def\thetalim{\theta^\star}

\newcommand{\rmd}{\mathrm{d}}
\def\funcAw{\mathbf{A}}

\def\funcbw{\mathbf{b}}
\newcommand{\funcb}[1]{\funcbw(#1)}
\newcommandx{\zmfuncA}[2][1=]{\widetilde{\funcAw}_{#1}(#2)}
\newcommandx{\zmfuncAw}[1][1=]{\widetilde{\funcAw}_{#1}}
\newcommandx{\zmfuncbw}[1][1=]{\widetilde{\funcbw}_{#1}}
\newcommandx{\zmfuncb}[2][1=]{\widetilde{\funcbw}_{#1}(#2)}
\def\funnoisew{\varepsilon}
\newcommandx{\funcnoise}[2][2=]{\funnoisew_{#2}(#1)}
\newcommandx{\funcnoiselim}[2][2=]{\omega_{#2}(#1)}
\newcommandx{\avgfuncnoiselim}[1]{\bar\omega_{#1}}

\newcommandx{\funcct}[2][1=]{\funcctilde^{#1}(#2)}

\newcommand{\trace}[1]{\operatorname{Tr}(#1)}

\def\State{Z}

\newcommand{\pscal}[2]{\langle#1\,,\,#2\rangle}

\newcommandx{\CovC}[1][1=u]{\operatorname{C}_{#1}}

\def\msz{\mathsf{Z}}
\def\mcz{\mathcal{Z}}

\DeclareMathAlphabet{\mathpzc}{OT1}{pzc}{m}{it}

\def\lyapW{\mathpzc{W}}

\newcommandx{\bias}[1][1=\alpha]{\operatorname{B}_{#1}}

\newcommandx\probaMarkovTilde[2][2=]
{\ifthenelse{\equal{#2}{}}{{\widetilde{\mathbb{P}}_{#1}}}{\widetilde{\mathbb{P}}_{#1}\left[ #2\right]}}

\def\mcf{\mathcal{F}}

\newcommandx{\bA}[1][1=]{\bar{\mathbf{A}}_{#1}}

\def\thetas{\thetalim}

\def\funcctilde{\tilde{c}_u}

\newcommandx{\barb}[1][1=]{\bar{\mathbf{b}}_{#1}}
\newcommandx{\driftb}[1][1=p]{\bar{b}_{#1}}

\newcommandx{\barA}[1][1=]{\bar{A}^{#1}}

\newcommandx{\boldb}[1][1={q}]{\mathsf{b}_{#1}}

\newcommandx{\ConstGW}[1][1={n,\lyapW}]{\operatorname{G}_{#1}}

\newcommandx{\ConstMW}[1][1={n,\lyapW}]{\operatorname{M}_{#1}}

\Crefname{assumptionC}{\textbf{C}\hspace{-1pt}}{\textbf{C}\hspace{-1pt}}
\crefname{assumptionC}{\textbf{C}}{\textbf{C}}

\Crefname{assumptionM}{\textbf{UGE}\hspace{-1pt}}{\textbf{UGE}\hspace{-1pt}}
\crefname{assumptionM}{\textbf{UGE}}{\textbf{UGE}}

\Crefname{assumptionH}{\textbf{H}\hspace{-1pt}}{\textbf{H}\hspace{-1pt}}
\crefname{assumptionH}{\textbf{H}}{\textbf{H}}

\def\distance{\mathsf{d}}

\newcommandx{\vartconstwas}[1][1=V]{c_{#1}}

\newcommandx{\deltawas}[1][1=*]{\delta_{#1}}

\newcommandx{\wasser}[4][1=\distance,4=]{\mathbf{W}_{#1}^{#4}\left(#2,#3\right)}
\newcommandx{\covcoeff}[2]{\rho_{#1}^{(#2)}}

\newcommand{\dobrush}{\mathsf{\Delta}}
\newcommandx{\dobru}[3][1=,3=]{\dobrush_{#1}^{#3}( #2)}  %

\def\invariantQ{\pi}

\def\eqdef{:=}

\def\Markov{\mathrm{M}}

\newcommandx{\dlim}[1]{\ensuremath{\stackrel{#1}{\Longrightarrow}}}

\def\nagent{N}
\def\nrounds{T}

\def\nlupdates{H}
\newcommandx{\ContMat}[3][1=\step,2=\nlupdates,3=c]{\operatorname{C}_{#1,#2}^{#3}}

\newcommandx{\nfuncbw}[1][1=]{\funcbw_{{#1}}}
\newcommandx{\nfuncAw}[1][1=]{\funcAw_{{#1}}}
\newcommandx{\nfuncb}[2][1=]{\nfuncbw[#1](#2)}
\newcommandx{\shnfuncb}[2][1=]{\nfuncbw[#1]^{#2}}
\newcommandx{\nfuncA}[2][1=]{\nfuncAw[#1](#2)}
\newcommandx{\shnfuncA}[2][1=]{\nfuncAw[#1]^{#2}}

\newcommandx{\nbarA}[1][1=]{\bar{\mathbf{A}}_{{#1}}}
\newcommandx{\nbarb}[1][1=]{\bar{\mathbf{b}}_{{#1}}}

\newcommandx{\nzmfuncA}[3][1=,2=]{\tilde{\funcAw}^{#1}_{#2}(#3)}
\newcommandx{\nzmfuncb}[3][1=,2=]{\tilde{\funcbw}^{#1}_{#2}(#3)}

\def\step{\eta}
\def\mcf{\mathcal{F}}

\def\controlvar{\xi}

\newcommandx{\CPEcontrolvar}[2]{\CPE{\controlvar_{#1}^{#2}}{\mcf_{#1-1}}}

\newcommandx{\cvrflg}[1]{\bar{\phi}_{#1}}
\newcommandx{\cvrfll}[1]{\bar{\psi}_{#1}}

\newcommandx{\Cstepmat}[1][1=c]{C_{\step,\nlupdates}^{{#1}}}
\newcommandx{\Cmat}[2][1=c]{C_{#2}^{{#1}}}
\newcommandx{\dCmat}[2][1=c]{\widetilde{C}_{#2}^{{#1}}}
\newcommandx{\fldCmat}[2][1=c]{\widetilde{C}_{#2}^{{#1}}}
\newcommandx{\bCstepmat}[1][1=c]{\bar{C}_{\step,\nlupdates}^{{#1}}}

\newcommandx{\Prod}[2][1=]{\Gamma_{#2}^{#1}}
\newcommandx{\dProd}[2][1=]{\widetilde{\Gamma}_{#2}^{#1}}
\newcommandx{\flProd}[2][1=]{\Gamma_{#2}^{\text{(fl)}}}

\newcommandx{\avgfuncnoise}[1]{\bar \varepsilon^{#1}}
\newcommandx{\nfuncnoise}[2][2=]{\varepsilon_{#2}^{{#1}}}
\newcommandx{\ndfuncnoise}[2][2=]{\widetilde{\varepsilon}_{#2}^{{#1}}}

\newtheorem{assumTD}{\textbf{TD}\hspace{-1pt}}
\Crefname{assumTD}{\textbf{TD}\hspace{-1pt}}{\textbf{TD}\hspace{-1pt}}
\crefname{assumTD}{\textbf{TD}}{\textbf{TD}}

\def\mixingparam{\beta}

\newcommand{\grb}[3]{\mathrm{g}_{#1}^{#2}(#3)}
\newcommand{\grbmix}[1]{\mathrm{g}_{\text{mix}}^{#1}}

\newcommandx{\projerror}[2]{e^{#2}_{#1}}
\newcommandx{\avgprojerror}[1]{\bar{e}^{#1}}
\newcommandx{\avgerroragent}[1]{E^{#1}}
\newcommandx{\disagreement}[2]{d^{#2}_{#1}}
\newcommandx{\normdisagreement}[1]{D^{#1}}

\newcommandx{\Lyapunov}[1][1=]{\Phi^{#1}}

\newcommand{\term}[1]{\mathbf{(#1)}}

\def\hgty{\zeta}

\usepackage[textsize=tiny]{todonotes}

\usepackage{hyperref}
\usepackage{url}

\title{Best of Both Worlds in Federated LSA: Speedup When Possible, Personalization Always}

\author{Safwan Labbi\textsuperscript{1}, Paul Mangold\textsuperscript{1}, Eric Moulines\textsuperscript{2,3} \\
\textsuperscript{1}  CMAP, CNRS, École Polytechnique, Institut Polytechnique de Paris, 91120 Palaiseau, France
\\
\textsuperscript{2} Mohamed bin Zayed University of Artificial Intelligence, UAE 
\\
\textsuperscript{3} LRE EPITA , 94270 Le Kremlin-Bicêtre, France
\\
\texttt{\{safwan.labbi, paul.mangold\}@polytechnique.edu}
\\
\texttt{eric.moulines@mbzuai.ac.ae}
}

\iclrfinalcopy %
\begin{document}

\maketitle

\fancyhead[L]{Preprint}

\begin{abstract}
We study personalized federated linear stochastic approximation (LSA), a framework which notably encompass personalized temporal difference learning. In this setting, heterogeneous agents collaborate to solve distinct linear fixed-point equations, each corresponding to an agent-specific learning problem. A central open question in personalized learning is whether a single method can adapt to an unknown level of heterogeneity by converging to each agent's personalized solution in all regimes while achieving a linear speedup in the number of agents when their learning problems are sufficiently similar. We answer this question affirmatively by introducing \perlsa, a minimalist algorithm that mixes each agent's local stochastic update with the average update across agents, at no additional computational cost relative to standard federated methods. We prove that \perlsa\, achieves \emph{best-of-both-worlds} guarantees without any prior knowledge on the level of heterogeneity. Our analysis is based on a sharp decomposition of the error into consensus and disagreement components. The consensus error decays rapidly, whereas the disagreement error decays more slowly but becomes negligible in low-heterogeneity regimes.
\end{abstract}

\section{Introduction}
\label{Introduction}
Linear stochastic approximation (LSA) is a fundamental problem in stochastic optimization, where the goal is to solve a system of linear equations whose coefficients are only available through stochastic estimates.
It encompasses problems from optimization \citep{robbins1951stochastic,defossez2015averaged}, control and telecommunications \citep{benveniste1990adaptive,borkar2008stochastic}, and reinforcement learning \citep{tsitsiklis1996analysis,sutton2018reinforcement}. 
In the presence of noise, solving LSA problems generally requires a large number of samples \citep{bhandari2018finite}. This may quickly become impractical in real-world settings where data collection is costly or distributed across multiple sources. 

Federated stochastic approximation \citep{doan2019FiniteTime} has been developed as a way of improving the sample efficiency of stochastic approximation through collaboration.
The promise of this paradigm is to accelerate the training by allowing agents to benefit from the experience of other agents.
This provides a natural way of reducing overall sample complexity. 
Most federated approaches, such as Federated Averaging (FedAvg) \citep{mcmahan2017communication}, rely on averaging local updates to learn a single shared model. 
While this is effective when agents are homogeneous, the performance can degrade significantly in the presence of heterogeneity across agents, where one model cannot perfectly fit all agents \citep{kairouz2021advances,zhao2018federated}.

To address this limitation, a growing body of work has focused on \emph{personalized federated learning}, where the goal is to exploit data from other agents while still learning agent-specific solutions. Many approaches rely on structural assumptions to capture similarities between agents, such as clustering \citep{sattler2020clustered}, multi-task learning formulations \citep{smith2017federated,mansour2020three}, or regularization-based methods \citep{li2020federated,t2020personalized}. While effective in certain regimes, these methods typically depend on prior assumptions about the structure of heterogeneity. More recent approaches aim to promote collaboration among sufficiently similar agents, but they rely on prior knowledge of inter-agent similarities, which is unavailable in practice \citep{philippenko2025adaptive}.

In this work, we take a different perspective and propose a surprisingly simple approach.
Instead of attempting to identify which agents should collaborate, we allow all agents to share information through a lightweight mixing of their stochastic updates.
This approach is inspired by a form of \emph{inclusionism}, where every contribution is accepted on the premise that useful ones will ultimately dominate less effective ones over time.
To this end, we introduce \perlsa\, (Personalized Federated Linear Stochastic Approximation), where agents use a linear combination of their own update together with the averaged update of all agents.
Crucially, this is done without any consideration on the (non-)homogeneity of the agents.
The algorithm is thus very minimalist, as it simply combines local updates with their global average, without imposing any explicit structure on agent similarity.

We show that, despite its simplicity, this approach is effectively able to exploit all available data while preserving personalization.
Our contributions are summarized as follows:
\begin{itemize}[leftmargin=*]
    \item We introduce \perlsa, a novel personalized federated method which mixes updates across all agents, interpolating between fully local learning and classical federated averaging \citep{mcmahan2017communication}, with the relative importance of each part being determined by a single hyperparameter.
    
    \item We show that \perlsa\, converges to the optimal solution of each individual agent, ensuring full personalization even when incorporating global information from other agents.
    
    \item We derive non-asymptotic convergence rates based on a decomposition into consensus and disagreement components. The consensus error decreases rapidly, while the disagreement term decreases more slowly but is smaller in more homogeneous settings.
    
    \item We show that, under mild heterogeneity, \perlsa\, enables each agent to achieve a linear speedup.
    
\end{itemize}
Our results demonstrate that a minimal mixing strategy is sufficient to achieve the \emph{best of both worlds}: full personalization together with the benefits of collaboration, without relying on restrictive assumptions about agent similarity.
We first discuss related work in \Cref{sec:related-work}, and give context on LSA in \Cref{sec:federated-lsa}. Then, we introduce our method and analyze it in \Cref{sec:analysis-fed-lsa-iid}.
We then instantiate these results for TD learning in \Cref{sec:perso-td}. \Cref{sec:numerical} studies the numerical behaviour of our method.
Finally, we conclude and discuss perspectives in \Cref{sec:conclu}.

\section{Related Work}
\label{sec:related-work}

\paragraph{Federated LSA and Federated Reinforcement Learning.}
Single-agent Linear Stochastic Approximation (LSA) has been widely studied \citep{robbins1951stochastic,tsitsiklis1996analysis,bhandari2018finite}. 
Federated approaches for LSA have also been proposed, generally with the aim of learning a \emph{single} shared solution from multiple agents' data.
Standard FedAvg \citep{mcmahan2017communication,kairouz2021advances} extends to LSA by averaging local updates. 
In heterogeneous settings, FedAvg can diverge or bias the solution \citep{zhao2018federated}, and local training methods with heterogeneity bias correction still optimize a common global objective.
Recent work \citep{khodadadian2022federated,wang2024federated,pmlr-v258-labbi25a,pmlr-v300-labbi26a,mangold2025convergence} considers multiple agents collecting data under possibly different environments.  For example, \cite{khodadadian2022federated} propose federated versions of TD and $Q$-learning and prove linear convergence speedups in the number of agents under Markovian noise.  \cite{wang2024federated} study federated TD learning with heterogeneous environments, aggregating local updates to converge to the optimal value function of the \emph{mixture} environment.  Similarly, Mangold \etal~\citep{mangold2024scafflsa} analyze the FedLSA algorithm (a FedAvg-like method for linear fixed-point problems) and propose SCAFFLSA, which uses control variates to reduce client drift and achieve faster communication complexity.  
These federated LSA/RL methods demonstrate that multiple agents can accelerate learning \citep{khodadadian2022federated,wang2024federated}.  However, they target a \emph{single} global or mixture solution and do not address the fully personalized case.  In short, existing federated LSA approaches focus on global objectives (or mixture objectives) and either ignore or correct heterogeneity in complex ways.  By contrast, \perlsa\, is designed to converge to each agent’s own solution, while still attaining a linear speedup through simple gradient mixing.

\paragraph{Personalized Federated Learning.}
Personalized federated learning departs from federated averaging (FedAvg) by learning models adapted to each agent. 
A common approach is \emph{multi-task} learning, where agents have related but distinct objectives. \citet{smith2017federated,mansour2020three} formulate a problem where each agent learns its own model while sharing information through a shared representation or regularizer. 
A second approach is \emph{cluster-based}, assuming that clients can be partitioned into clusters with similar data distributions; for instance, clustered FL \citep{sattler2020clustered} alternates between grouping clients and training cluster-specific models.
Other approaches, such as \citep{li2020federated,t2020personalized} add personalization terms to standard federated optimization, optimizing a global model plus a personalization part, controlling the trade-off between the two with regularizers.
These methods improve over FedAvg in non-IID settings \citep{zhao2018federated,kairouz2021advances}, but they rely on additional assumptions: multi-task and clustering approaches assume either low-rank/latent structure or well-separated client clusters \citep{smith2017federated,sattler2020clustered}, which may not hold in practice. 
Regularization-based methods (e.g.\ \citep{li2020federated,t2020personalized}) still implicitly optimize a global objective and generally require tuning of penalty hyperparameters.  More recent approaches propose control-variate-style methods that allow each agent to converge to its own local solution, but require knowledge of pairwise agent data-distribution ratios, which are unavailable in practice \citep{zhang2025personalized}.
In contrast, our approach does not require such assumptions and simply aggregates updates from all agents, resulting in a remarkably simple algorithm.

\section{Background on Personalized Federated Linear Stochastic Approximation}
\label{sec:federated-lsa}
In personalized federated linear stochastic approximation, 
we consider a set of $\nagent$ agents, each of which has its own individual system of linear equations.
The goal is to solve for each system individually, aiming to find a $\thetas_c$ for each agent $c \in [\nagent]$ such that 
\begin{equation}
\label{eq:lsa_syst_fed}
\textstyle
 \bA[c] \thetas_c 
 = 
 \barb[c]
 \eqsp,
\end{equation}
where for $c \in [\nagent]$, $\bA[c] \in \rset^{d \times d}$, $\barb[c] \in \rset^d$.
Note that due to differences in each agent’s local problem, the values of $\bA[c]$'s and $\barb[c]$'s can be different from each other.
Following prior work on LSA, we assume that each $\bA[c]$ satisfies the following contraction condition \cite{mou2020linear}.
\begin{assum}
\label{assum:exp_stability}
There exist $a > 0$, $\step_{\infty} > 0$, such that $\step_{\infty} a \leq 1/2$, and for $\step \in (0;\step_{\infty})$, $c \in [\nagent]$, $u \in \rset^{d}$, it holds for $\State_0^c \sim \invariantQ_c$, that $\PE^{1/2}\bigl[ \normop{(\Id - \step \nfuncA[c]{\State_0^c})u}^{2}\bigr] \leq (1 - \step a) \norm{u}$.
\end{assum}
Under this assumption, there is a unique solution $\thetas_c$ for each linear system.
In personalized federated LSA, neither matrices $\bA[c]$ nor vectors $\barb[c]$ are observed directly.
Instead, each agent $c \in [N]$ has access to its own observation sequence $(\State_{k}^c)_{k \in \nset}$, which are independent across agents.
Agent $c$ obtains estimates $\{(\nfuncA[c]{\State_{k}^c}, \nfuncb[c]{\State_{k}^c})\}_{k \in \nset }$ of $\bA[c]$ and $\barb[c]$, where $\nfuncAw[c]: \msz \to \rset^{d \times d}$ and $\nfuncbw[c] : \msz \to \rset^d$ are two measurable functions.
We also define the error of estimation of $\nbarA[c]$ and $\nbarb[c]$ as
$\smash{\zmfuncb[c]{z} = \nfuncb[c]{z} - \nbarb[c]}$, 
$\smash{\zmfuncA[c]{z} = \nfuncA[c]{z} - \nbarA[c]}$.
This allows us to measure the noise at local solutions as
\begin{align}
    \label{eq:def-epsilon-noise}
    \funcnoise{z}[c]=  \zmfuncA[c]{z} \thetalim_c - \zmfuncb[c]{z}
    \eqsp, \quad \noisecov_c & \!=\!\!
\int_{\Zset} \funcnoise{z}[c]\funcnoise{z}[c]^\top \rmd \invariantQ_c(z)
\eqsp,
\end{align}
which are finite whenever the following assumption on $\{\State_c^{t}\}_{t \ge 0}$ holds.
\begin{assum}
\label{assum:noise-level-flsa}
For each agent $c$, $(Z_{c}^k)_{k \in \nset}$ are \iid\ random variables with values in $(\msz,\mcz)$ and distribution $\invariantQ_c$ satisfying $\PE_{\invariantQ_c}[\nfuncA[c]{\State_{c}^k}] = \bA[c]$ and $\PE_{\invariantQ_c}[ \funcb{\State_{c}^k} ] = \barb[c]$, 
and we assume
\begin{align}
\nonumber
\bConst{A} = \max_{c \in [\nagent]}\Big\{ \max\Big(\sup_{z \in \msz}\normop{\nfuncA[c]{z}}, \sup_{z \in \msz}\normop{\zmfuncA[c]{z}}\Big)  \Big\} < \infty \eqsp, \quad \max_{c \in [\nagent]} \PE[ \left\|\varepsilon_{c}(\State_{c}^1) \right\|^2 ] < \infty \eqsp.
\end{align}
\end{assum}
This assumption will allow us to carefully track the impact of the noise in the computations.
Finally, we define the following quantities, which measure the stochasticity and heterogeneity of the problem,
\begin{align}
\label{def:heterogeneity_agents}
\avgnoise^2 =\frac{1}{\nagent}
\sum_{c=1}^{\nagent}\trace{\noisecov_c} \eqsp, \quad \hgty = \max_{c \in \nagent} \left\| \bA[c] - \bA\right\|
\eqsp, \quad \Delta_{\star}  =
\frac{1}{\nagent}
\sum_{c=1}^{\nagent}
\norm{
\thetas_c
-
\frac{1}{\nagent}
\sum_{c'=1}^{\nagent}
\thetas_{c'}
}^2 \eqsp.
\end{align}
which we will use in our analysis to track noisy updates and to quantify the impact of heterogeneity.

\section{The \texorpdfstring{\perlsa~}{PF-LSA}Algorithm}
\label{sec:analysis-fed-lsa-iid}

We now introduce and analyze \perlsa, a new algorithm designed to solve the personalized federated LSA problem.
At each round $t \geq 0$, every agent $c \in [\nagent]$ maintains a local parameter $\theta_c^t$ and receives a fresh observation $\State_c^{t+1}$.
Using this observation, the agent computes the local stochastic update $
\nfuncA[c]{\State_c^{t+1}}\theta_c^t
-
\nfuncb[c]{\State_c^{t+1}},$
and sends it to the central server.
The server then averages the local updates and broadcasts the resulting global update back to all agents.
Each agent combines this averaged update with its own local update through the mixing parameter $\mixingparam \in (0,1)$, leading to the recursion
\begin{align}
\theta_c^{t+1}
=
\theta_c^t
&
\textstyle
- 
\step \mixingparam
\left(
\nfuncA[c]{\State_c^{t+1}}\theta_c^t
-
\nfuncb[c]{\State_c^{t+1}}
\right)
\nonumber -
\frac{\step(1-\mixingparam)}{\nagent}
\sum_{c'=1}^{\nagent}
\left(
\nfuncA[c']{\State_{c'}^{t+1}}\theta_{c'}^t
-
\nfuncb[c']{\State_{c'}^{t+1}}
\right)
\eqsp .
\end{align}
The parameter $\mixingparam$ controls the relative importance of the local update in comparison with the global update: the first term favors local learning, while the second term enables collaboration by incorporating information from all agents.
We give the pseudo-code of \perlsa~in \Cref{alg:personalized_lsa}. 

Next, in \Cref{sec:asymptotic-rates}, we motivate the mixing mechanism by showing that the limit point of \perlsa~ is the collection of local solutions, ensuring personalization.
We then derive non-asymptotic rates in \Cref{sec:non-asymptotic-rates}: a Lyapunov argument first gives a general but slow rate, which we then refine through a consensus/disagreement decomposition of the error, uncovering a faster convergence rate for the consensus.
These results show that, for any positive mixing parameter, \perlsa~ converges to the true local solutions, and that in suitable heterogeneous regimes it achieves personalization together with a linear speed-up from collaboration.

\begin{algorithm}[t]
\caption{\perlsa: Personalized Federated Linear Stochastic Approximation}
\label{alg:personalized_lsa}
\begin{algorithmic}
\STATE \textbf{Input: } $\step > 0$; $0 < \mixingparam< 1$; $\nrounds, \nagent > 0$; $\theta_1^{0}, \dots, \theta_{\nagent}^{0} \in \rset^d$
\FOR{$t=0$ to $\nrounds-1$}
\FOR{$c=1$ to $\nagent$}
\STATE Receive $\State_c^{t+1}$, compute:
$\grb{c}{\State_{c}^{t+1}}{\theta_{c}^{t}} \!=\! \nfuncA[c]{\State_{c}^{t+1}} \theta_c^{t} \!-\! \nfuncb[c]{\State_c^{t+1}}$ and send it to the central server
\ENDFOR
\STATE Server computes $\grbmix{t} = \frac{1}{\nagent} \sum_{c'=1}^{\nagent} \grb{c}{\State_{c'}^{t+1}}{\theta_{c'}^{t}}$ and broadcasts it to all agents
\FOR{$c=1$ to $\nagent$}
\STATE Agent updates its model: $\theta_c^{t+1} = \theta_{c}^{t} - \step \mixingparam \grb{c}{\State_{c}^{t+1}}{\theta_{c}^{t}} - \step (1-\mixingparam) \grbmix{t} $ 
\ENDFOR
\ENDFOR
\STATE Return $\{ \theta_c^{\nrounds} \}_{c=1}^\nagent$
\end{algorithmic}
\end{algorithm}

\subsection{\texorpdfstring{\perlsa\,}{PF-LSA}'s Limit Points are the Solutions of Individual Systems}
\label{sec:asymptotic-rates}

To establish the limit point of the \perlsa\, algorithm, we identify the set of parameters that cancels out the deterministic update of the algorithm.
By definition, any limit point of this algorithm $(\theta_{1}^{\infty}, \dots, \theta_{\nagent}^{\infty})$ must be a fixed-point, meaning that they must satisfy the stationary conditions
\begin{align}
  \label{eq:fixed-point}
  \theta_c^{\infty}
  & =
    \theta_c^{\infty}
    - \step \mixingparam \bA[c] ( \theta_c^{\infty} - \thetas_{c} )
    - \tfrac{\step (1 - \mixingparam)}{\nagent} \textstyle \sum_{c'=1}^\nagent \bA[c'] (\theta_{c'}^{\infty} - \thetas_{c'})
                   \eqsp, \quad \forall c\in [\nagent]\eqsp.
\end{align}
Solving these equations gives the following result, which shows that if \perlsa\, converges, then its limit must coincide with each agent’s optimal solution.
\begin{theorem}
  \label{thm:convergence-asympt-loc-solutions}
  Let $\{ \theta_c^t : c \in [\nagent], t \ge 0 \}$ be the sequence of parameters of the \perlsa\, algorithm with deterministic updates, $\step \le \step_\infty$ and $0 < \mixingparam \le 1$. Additionally assume that there exists $(\theta_{1}^{\infty}, \dots, \theta_{\nagent}^{\infty})$ such that $(\theta_1^t, \dots, \theta_{\nagent}^t) \xrightarrow[t\to\infty]{}(\theta_{1}^{\infty}, \dots, \theta_{\nagent}^{\infty})$.
  Then, for all $c \in [\nagent]$, we have that
  \begin{align}
    \lim_{t \rightarrow \infty} \theta_c^t = \thetas_c
    \eqsp.
  \end{align}
\end{theorem}
\begin{proof}
  The proof follows from solving \eqref{eq:fixed-point}. Indeed, averaging \eqref{eq:fixed-point} over all the agents gives $
\textstyle
\sum_{c'=1}^\nagent \bA[c'] (\theta_{c'}^{\infty} - \thetas_{c'}) = \mathbf{0}$.
Plugging in the previous identity in \eqref{eq:fixed-point} gives $
\bA[c] ( \theta_c^{\infty} - \thetas_{c} ) = \mathbf{0}$
and the result follows from the uniqueness of the solution of the linear system $\bA[c] \theta_c^{\infty} = \barb[c]$
\end{proof}
This theorem highlights the core idea behind \perlsa: as long as $\mixingparam > 0$, meaning that local parameters are given a slightly increased importance compared to the others, each agent \emph{converges to its individual solution}.
Notably, this holds for any value of $\mixingparam$, highlighting the very different behavior of the method in the cases $\mixingparam = 0$ and $\mixingparam > 0$, and the fact that, as $\beta \rightarrow 0$, \perlsa\, does not boil down to FedAvg.
This phenomenon will allow us to choose very small values of $\mixingparam$, which will be helpful for controlling the variance of the updates.
Importantly, this will allow \perlsa\, to achieve linear speed-up, and fast convergence, which we study in the remaining part of this section.

\subsection{Non-Asymptotic Convergence Rates}
\label{sec:non-asymptotic-rates}

\paragraph{Lyapunov.}

We now establish a non-asymptotic convergence rate of the \perlsa\, algorithm.
Precisely, we start by proving that \perlsa\, converges towards each agent’s personalized solution by analyzing a Lyapunov function tailored to the algorithm's updates.
For any $\lambda > 0 $, we define
\begin{align}
  \label{eq:def-lyap-main}
  \Lyapunov[t]_{\lambda} & 
\textstyle= \PE\left[\norm{ \frac{1}{\nagent} \sum_{c'=1}^\nagent (\theta_{c'}^t - \thetas_{c'})}^2
                 + \mixingparam \lambda \frac{1}{\nagent} \sum_{c=1}^{\nagent} \norm{ \theta_c^t - \thetas_c - \frac{1}{\nagent} \sum_{c'=1}^\nagent (\theta_{c'}^t - \thetas_{c'})}^2\right]
                 \eqsp.
\end{align}
This Lyapunov function considers the agents' local iterates by splitting the error into two parts: a consensus part, controlling the average of errors $ \smash{\frac{1}{\nagent} \sum_{c'=1}^\nagent (\theta_{c'}^t - \thetas_{c'})}$ and a disagreement part, which controls the difference between agents' local error and averaged error.
Next, we show that, under suitable conditions on $\lambda$, the Lyapunov function contracts and converges to zero.
\begin{theorem}
\label{cor:convergence_rate_lyapunov_main}
Assume \Cref{assum:exp_stability} and \Cref{assum:noise-level-flsa}.  Let
$0<\mixingparam\leq1$, and let $
\step\in\left(0,\min\left\{\step_\infty,
\frac{a}{16\bConst{A}^2}\right\}\right)$.
For any $0<\lambda\leq1$ such that
$(1-\lambda)\hgty
\leq \frac{a\sqrt{\lambda}}2,$ and $
\frac{\step\bConst{A}^2}{a}
\frac{\max\{\mixingparam,\lambda\}}{\lambda}
\leq\frac16$, we have for every $t\geq0$ that
\begin{align}
\nonumber
\textstyle
\PE\left[
\Lyapunov[t]_{\lambda} 
\right]
\leq
(1-\step\mixingparam a)^t
\Lyapunov[0]_{\lambda}  +
\frac{2\step}{a}
\left[
\mixingparam\left(1-\frac1{\nagent}\right)
+\frac{\lambda}{\nagent}
\right]\avgnoise^2.
\end{align}
\end{theorem}
The conditions on $\lambda$ can always be satisfied: for instance, setting
$
\lambda
=
\min\{
1,\;
\max\{
4\frac{\hgty^2}{a^2},
\;
6\frac{\step\mixingparam\bConst{A}^2}{a}
\}
\}
$
satisfies both conditions. The previous bound shows that \perlsa\, indeed converges toward each agent's individualized solution \emph{for any level of heterogeneity}. Although this result alone is not sufficient to establish linear speedup, this result provides a crucial ingredient for establishing linear speedup. Indeed, it allows us to control the disagreement component, which we combine with a refined analysis of the consensus dynamics below to derive sharper convergence guarantees and establish linear speedup in low-heterogeneity regimes.

\paragraph{Refined Convergence Bound.}
To better capture the dynamics of \perlsa, we decompose the average MSE over agents into the consensus component $\avgprojerror{t}$ and the disagreement component $\normdisagreement{t}$, %
\begin{align}
\textstyle
\projerror{c}{t} =\theta_{c}^{t} - \theta_{c}^{\star} \eqsp, \quad  \avgprojerror{t} = \frac{1}{\nagent} \sum_{c=1}^{\nagent} \theta_{c}^{t} - \theta_{c}^{\star} \eqsp, \quad 
\disagreement{c}{t} = (\theta_c^t - \thetas_c) - \avgprojerror{t}
\eqsp, \quad 
\normdisagreement{t} = \frac{1}{\nagent} \sum_{c=1}^{\nagent} \| \disagreement{c}{t} \|_{2}^2
\eqsp.
\end{align} 
Subsequently, we show that this decomposition reveals the two-clock structure of \perlsa: a fast consensus clock and a slow disagreement clock. Next, we derive a recursion on the consensus mode.
\begin{lemma}[Fast average recursion]
\label{thm:convergence_rate_avg_error_main}
Assume \Cref{assum:exp_stability} and \Cref{assum:noise-level-flsa}.  Let $
\step\in\left(0,\min\left\{\step_\infty,
\frac{a}{16\bConst{A}^2}\right\}\right)$, and $0<\mixingparam\leq1$, for every
$t\geq 0$, we have
\begin{align}
\nonumber
\textstyle
\PE\left[
\|\avgprojerror{t+1}\|_2^2\mid\filtr^t
\right]
\leq
(1-\step a)\|\avgprojerror{t}\|_2^2
+
\left(
2\step^2\bConst{A}^2+
\frac{4\step\hgty^2}{a}
\right)\normdisagreement{t}
+\frac{2\step^2}{\nagent}\avgnoise^2.
\end{align}
\end{lemma}
The first term in Lemma \ref{thm:convergence_rate_avg_error_main}, $
(1-\step a)^t \|\avgprojerror{0}\|_2^2,
$
corresponds to the consensus dynamics and contracts at the same rate as the initial error-forgetting term in single-agent LSA. 
Moreover, the variance term is of order $\step \avgnoise^2/(a\nagent)$, showing a linear variance reduction in the number of agents. The remaining term $(
2\step^2\bConst{A}^2+
\frac{4\step\hgty^2}{a})\normdisagreement{t}$ is the price paid for personalization and heterogeneity. 
This term measures how much the consensus is perturbed by disagreement and by heterogeneity of the local operators. 
By combining this recursion with \Cref{cor:convergence_rate_lyapunov_main}, we derive the following bound on the average MSE.
\begin{theorem}[Unified all-heterogeneity bound]
\label{cor:convergence_rate_individual_error_simplified_main} Assume \Cref{assum:exp_stability} and \Cref{assum:noise-level-flsa}.  Let $
\step\in\left(0,\min\left\{\step_\infty,
\frac{a}{16\bConst{A}^2}\right\}\right)$ and $0<\mixingparam\leq1$. For any
 $\lambda>0$ such that
$(1-\lambda)\hgty
\leq \frac{a\sqrt{\lambda}}2,$ and $
\frac{\step\bConst{A}^2}{a}
\frac{\max\{\mixingparam,\lambda\}}{\lambda}
\leq\frac16$, for every $t\geq0$,
\begin{align}
\textstyle
\nonumber
\frac1{\nagent}\sum\limits_{c=1}^{\nagent}
\PE\|\projerror{c}{t}\|_2^2
\lesssim
(1-\step a)^t
\|\avgprojerror{0}\|_2^2 \!+\!
[
1+
\frac{\hgty^2}{a^2}
]
(
1-\frac{\step\mixingparam a}{2}
)^t
(
\normdisagreement{0}
+\mixingparam\lambda
\|\avgprojerror{0}\|_2^2
)
+
\frac{2\step}{a}
(
1+
\frac{\hgty^2}{a^2}
)
[
\mixingparam
+\frac{1}{\nagent}
]
\avgnoise^2.
\end{align}
\end{theorem}
Theorem~\ref{cor:convergence_rate_individual_error_simplified_main} makes the two-clock structure of \perlsa\, explicit: a fast consensus clock $(1-\step a)^t$ and a slower personalization clock $(1-\step\mixingparam a/2)^t$. When heterogeneity is large, the slow mode dominates and \perlsa\, recovers the single-agent LSA rate; when heterogeneity is small, the slow mode becomes negligible and the variance reduction $(\mixingparam+1/\nagent)$ enables a linear speedup while preserving personalization. To characterize more precisely the convergence guarantees in different heterogeneity regimes, we instantiate the bound with the following explicit admissible choice of $\lambda$:
$
\lambda
=
\min\left\{
1,\;
\max\left\{
4\frac{\hgty^2}{a^2},
\;
6\frac{\step\mixingparam\bConst{A}^2}{a}
\right\}
\right\},$
which allows us to characterize more precisely when collaboration is beneficial.
\paragraph{Low-Heterogeneity Regime.}
We first consider the regime in which the local operators are sufficiently similar, that is
$\hgty^2<a^2/4$. In this case, the admissible choice of $\lambda$ introduced above does not saturate at $1$. Instantiating Theorem~\ref{cor:convergence_rate_individual_error_simplified_main} with this choice of $\lambda$ yields the following simplified bound.
\begin{corollary}[Low-heterogeneity regime]
\label{cor:unified-low-heterogeneity-main}
Assume \Cref{assum:exp_stability}, \Cref{assum:noise-level-flsa} and assume that
$\hgty^2<a^2/4$.  Let $
\step\in\left(0,\min\left\{\step_\infty,
\frac{a}{16\bConst{A}^2}\right\}\right)$, 
Then, for every $0<\mixingparam\leq1$ and every $t\geq0$,
\begin{align}
\nonumber
\textstyle
\frac1{\nagent}\sum\limits_{c=1}^{\nagent}
\PE\|\projerror{c}{t}\|_2^2
& \textstyle\lesssim
(1-\step a)^t\|\avgprojerror{0}\|_2^2
+
(
1-\frac{\step\mixingparam a}{2}
)^t
[
\normdisagreement{0}
+
\mixingparam\frac{\hgty^2}{a^2}
\|\avgprojerror{0}\|_2^2
] +
\frac{\step}{a}
(
\mixingparam
+\frac1{\nagent}
)
\avgnoise^2
+
\frac{\step\mixingparam^2\bConst{A}^2}{a}
\|\avgprojerror{0}\|_2^2.
\end{align}
\end{corollary}
The resulting bound highlights the favorable effect of low heterogeneity. The fast consensus component still contracts at rate $(1-\step a)^t$, while the amplitude of the slow personalization component is controlled by $
\normdisagreement{0}
+
\mixingparam\frac{\hgty^2}{a^2}
\|\avgprojerror{0}\|_2^2.$ 
Hence, when the local operators are sufficiently similar and the discrepancy between the agents' optimal solutions is negligible compared with the initial consensus error, \perlsa\, can simultaneously achieve personalization and linear speedup. We formalize this favorable regime in the following result.

\begin{corollary}[Sample complexity in the low-heterogeneity regime]
\label{cor:low-heterogeneity-sample-complexity-main}
Assume \Cref{assum:exp_stability}, \Cref{assum:noise-level-flsa} and assume that
$\hgty^2<a^2/4$.
Let $\epsilon>0$ and choose 
$\mixingparam=1/\nagent$ and the step size such that $
\step
\leq
\min\left\{
\step_\infty,\;
\frac{a}{16\bConst{A}^2},\;
\frac{a\nagent\epsilon}{144\avgnoise^2},\;
\frac{a\nagent^2\epsilon}
{96\bConst{A}^2\|\avgprojerror{0}\|_2^2}
\right\} $.
Then \perlsa\, satisfies
$\frac1{\nagent}\sum_{c=1}^{\nagent}
\PE\|\projerror{c}{T}\|_2^2
\leq\epsilon$ provided
\begin{align}
\nonumber
\textstyle
T
\gtrsim
\frac1a
\max\left\{
\frac1{\step_\infty},
\frac{\bConst{A}^2}{a},
\frac{\avgnoise^2}{a\nagent\epsilon},
\frac{\bConst{A}^2\|\avgprojerror{0}\|_2^2}
{a\nagent^2\epsilon}
\right\}
\max\left\{
\log\left(
\frac{\|\avgprojerror{0}\|_2^2}{\epsilon}
\right),
\;
\nagent
\log\left(
\left(
\normdisagreement{0}
+
\frac{\hgty^2}{a^2\nagent}
\|\avgprojerror{0}\|_2^2
\right)
/\epsilon
\right)
\right\}.
\end{align}
\end{corollary}
This corollary shows that, when the heterogeneity level $\hgty$ is small and the initial disagreement $\normdisagreement{0}$, which coincides with $\Delta_{\star}$ when all agents have the same initialization and quantifies the average dispersion of the local optimal solutions around their mean, is negligible compared with the initial consensus error, that is when $\nagent
\log\left(
\left(
\normdisagreement{0}
+
\frac{\hgty^2}{a^2\nagent}
\|\avgprojerror{0}\|_2^2
\right)
/\epsilon
\right) \leq \log\left(
\frac{\|\avgprojerror{0}\|_2^2}{\epsilon}
\right)$, \perlsa, achieves a sample complexity of order $\widetilde{\mathcal{O}}(1/(\nagent\epsilon))$.This highlights a key property of \perlsa: \emph{when the local problems are sufficiently similar relative to the common initial error, the algorithm simultaneously achieves personalization and a linear speedup in the number of agents}. Intuitively, in this regime, the agents' updates remain sufficiently aligned for collaboration to reduce the effective variance, while the local component of the algorithm still drives each agent toward its own individualized solution. Importantly, this behavior emerges without any tuning of the mixing parameter and arises automatically whenever the problem lies in this favorable regime.

\paragraph{High-Heterogeneity Regime.}

We now consider the complementary regime in which the local operators are significantly different. In this case, the admissible choice of $\lambda$ introduced above saturates at
$\lambda=1$. Instantiating Theorem~\ref{cor:convergence_rate_individual_error_simplified_main} with $\lambda=1$ yields the following result.
\begin{corollary}[High-heterogeneity regime]
\label{cor:unified-high-heterogeneity-main}
Assume \Cref{assum:exp_stability}, \Cref{assum:noise-level-flsa} and assume that
$\hgty^2 \geq a^2/4$. Let $
\step\in\left(0,\min\left\{\step_\infty,
\frac{a}{16\bConst{A}^2}\right\}\right)$.
Then, for every $0<\mixingparam\leq1$ and every $t\geq0$, it holds that
\begin{align}
\nonumber
\textstyle
\frac1{\nagent}\sum_{c=1}^{\nagent}
\PE\|\projerror{c}{t}\|_2^2
\lesssim
(1-\step a)^t\|\avgprojerror{0}\|_2^2
+
\frac{\hgty^2}{a^2}
\left(
1-\frac{\step\mixingparam a}{2}
\right)^t
\left(
\normdisagreement{0}
+\mixingparam\|\avgprojerror{0}\|_2^2
\right) +
\frac{\hgty^2}{a^2}
\cdot\frac{\step}{a}
\left[
\mixingparam
+
\frac{1}{\nagent}
\right]
\avgnoise^2.
\end{align}
\end{corollary}
The above bound makes explicit what changes in the highly heterogeneous regime. The fast consensus component still contracts at rate $(1-\step a)^t$. However, the slow personalization component,
$
\frac{\hgty^2}{a^2}
\left(
1-\frac{\step\mixingparam a}{2}
\right)^t
\left(
\normdisagreement{0}
+\mixingparam\|\avgprojerror{0}\|_2^2
\right),$
is no longer negligible. Thus, unlike in the low-heterogeneity regime, collaboration does not necessarily translate into a linear improvement of the overall convergence rate. Importantly, however, this slower dynamics is precisely what allows each agent to keep adapting toward its own individualized solution. Hence, \emph{even under arbitrarily large heterogeneity, \perlsa\, retains personalization.}
The corresponding sample-complexity guarantee is given next.
\begin{corollary}[Sample complexity in the high-heterogeneity regime]
\label{cor:high-heterogeneity-sample-complexity-main}
Assume \Cref{assum:exp_stability}, \Cref{assum:noise-level-flsa} and assume that
$\hgty^2 \geq a^2/4$.
Let $\epsilon>0$ and choose the step size such that
$
\step
\leq
\min\left\{
\step_\infty,\;
\frac{a}{16\bConst{A}^2},\;
\frac{a^3\nagent\epsilon}
{168\hgty^2\avgnoise^2}
\right\}$ and set $\mixingparam= 1/\nagent$. 
Then \perlsa\, achieves $
\frac1{\nagent}\sum_{c=1}^{\nagent}
\PE\|\projerror{c}{T}\|_2^2
\leq\epsilon$
provided
\begin{align}
\nonumber
\textstyle
T
\gtrsim
\frac1a
\max\left\{
\frac1{\step_\infty},
\frac{\bConst{A}^2}{a},
\frac{\hgty^2\avgnoise^2}
{a^3\nagent\epsilon}
\right\}
\max\left\{
\log\left(
\frac{\|\avgprojerror{0}\|_2^2}{\epsilon}
\right),
\nagent
\log\left(
\hgty^2
\left(
\normdisagreement{0}
+\nagent^{-1}\|\avgprojerror{0}\|_2^2
\right)
/
(a^2\epsilon)
\right)
\right\}.
\end{align}
\end{corollary}
This result illustrates the other side of the best-of-both-worlds behavior. Setting $\mixingparam=1/\nagent$ still yields a $1/\nagent$ reduction in the stochastic variance term. However, the personalization dynamics now evolves on the slower time scale $
1-\frac{\step a}{2\nagent},$ 
which introduces a factor $\nagent$ in the transient complexity. Consequently, the variance reduction obtained through collaboration can be offset by the slower personalization clock. This behavior is natural: when agents are highly heterogeneous, one cannot expect collaboration to accelerate the learning process, since their local learning problems are substantially different. In such a regime, the best one can reasonably expect is to recover the statistical complexity of independent learning, that $\widetilde{\mathcal{O}}(1/\epsilon)$ which is the case for \perlsa. Thus, \emph{even when heterogeneity is large, \perlsa\, falls back to the statistical rate of individualized learning while still preserving convergence toward each agent's own solution}.

These results establish the desired best-of-both-worlds behavior of \perlsa: when the agents are sufficiently similar, collaboration yields linear speedup; when they are highly heterogeneous, this speedup  disappear, but personalization is preserved and the algorithm recovers the statistical behavior of single-agent learning, without requiring any prior knowledge of the heterogeneity level.

\section{Application to Personalized Federated TD learning}
\label{sec:perso-td}
\paragraph{Background on TD Learning.} In this section, we instantiate our results on Personalized Federated LSA for Personalized Federated Temporal Difference (TD) learning.
Consider $\nagent$ Markov Decision Processes $\{(\S, \A, \PP^{\text{MDP}}_c, r_{c}, \gamma)\}_{c \in [\nagent]}$ with shared state space $\S$, action space $\A$, and discounting factor \(\gamma \in (0,1)\).
Each agent $c \in [\nagent]$ has its own transition kernel $\PP^{\text{MDP}}_c$, where \(\PP^{\text{MDP}}_c(\cdot|s,a)\) specifies the  transition probability from state $s$ upon taking action $a$ for this specific agent, 
as well as its own reward function $r_{c} : \S \times \A \to [0,1]$. %

In Personalized Federated TD learning, each agent uses its own policy $\pi_c$ and aims to estimate its local value function.
The goal is to approximate $V_c^{\pi_c}$, defined as, for $s \in \S$ and $c \in [\nagent]$,
\begin{align}
\textstyle
V_{c}^{\pi_c}(s) = \PE\Big[\sum\nolimits_{k=0}^{\infty}\gamma^{k} r_{c}(S_{c}^k,A_{c}^k)\Big]
\eqsp,
\!\eqsp\text{ with }
S_{c}^0= s, 
A_{c}^k \sim \pi_{c}(\cdot | S_{c}^k), 
\text{ and }
S_c^k \sim \PP_c^{\text{MDP}}(\cdot| S_c^k,A_c^k)
\eqsp.
\end{align}
In the following, we aim to approximate $V_{c}^{\pi_c}(s)$ as a linear combination of features built using a mapping $\varphi: \S \to \rset^{d}$.
Formally, we look for $\theta \in \rset^d$ such that the function $\mathcal{\hat{V}}_{\theta}(s) = \varphi^\top(s) \theta$ properly estimates the true value.
For $c \in [\nagent]$, we denote by $\mu_c$  the invariant distribution over $\S$ induced by $\pi_c$ and the transition kernel $\PP^{\text{MDP}}_c$ of agent $c$. Our goal is to find a parameter $\thetas_{c}$ which is defined as the unique solution to the projected Bellman equation, see \cite{tsitsiklis:td:1997}, which defines the best linear approximation of $V_{c}^{\pi_c}$. This can be seen as a LSA problem \citep{patil2023finite,wang2023federated}, viewing the local optimum $\thetas_{c}$ as the solution of the system
$\bA[c] \thetas_{c} = \barb[c]$\eqsp, where
\begin{equation}
\label{eq:system_matrix}
\bA[c] = \PE_{s \sim \mu_c, s' \sim \PMDP_{c}^{\pi_c}(\cdot|s)} [\phi(s)\{\phi(s)-\gamma \phi(s')\}^{\top}]\eqsp, \quad \text{and} \quad
\barb[c] = \PE_{s \sim \mu_{c}, a \sim \pi_{c}(\cdot|s)}[\phi(s) r_c(s,a)]\eqsp.
\end{equation}
In practice, when computing value functions, the tuples $\{ (S_c^t, A_c^t, S_c^{t+1}) \}_{t\in \nset}$ are sampled i.i.d.
\newcounter{assumTDbak}
\setcounter{assumTDbak}{\value{assumTD}}
\begin{assumTD}
\label{assum:generative_model}
\textit{$(S_c^t,A_c^t,S_c^{t+1})$ are generated \iid\, with $S_c^t \sim \mu_{c}$, $A_c^t \sim \pi_c(\cdot|S_c^{t})$, $S_c^{t+1} \sim \PP_c^{\text{MDP}}(\cdot|S_c^{t},A_c^t)$}\eqsp.
\end{assumTD}
The generative model assumption \Cref{assum:generative_model} is common in TD learning~\citep{dalal:td0:2018, li2023sharp, patil2023finite,samsonov2023finite}.
We also require the following assumption on the feature design matrix $\covfeat_{c} = \PE_{\mu^c}[\varphi(S_c^0)\varphi(S_c^0)^{\top}] \in \mathbb{R}^{d \times d}$.
\begin{assumTD}
\label{assum:feature_design}
\textit{Matrices $\Sigma_c^{\varphi}$ are non-degenerate with the minimal eigenvalue $\minlambda = \min_{c \in [\nagent]}\lambda_{\min}(\Sigma_c^{\varphi}) > 0$. Moreover, the feature mapping $\varphi(\cdot)$ satisfies $\sup_{s \in \S} \|\varphi(s) \| \leq 1$}.
\end{assumTD}
This assumption ensures the uniqueness of the optimal parameter $\thetas_{c}$. Under \Cref{assum:generative_model} and \Cref{assum:feature_design} we check that the Personalized Federated LSA assumptions \Cref{assum:exp_stability}, and \Cref{assum:noise-level-flsa}  hold with the following values.
\begin{claim}
\label{prop:condition_check} %
Assume \Cref{assum:generative_model} and \Cref{assum:feature_design}. Then the sequence of TD(0) updates satisfies \Cref{assum:noise-level-flsa} and \Cref{assum:exp_stability} with
\begin{gather}
\bConst{A} \!= 1\!+\!\gamma \eqsp, \eqsp
\trace{\noisecov_c} \!\leq 2 (1\!+\!\gamma)^2 \left(\norm{\thetas_{c}}^2 \!+\! 1\right) \eqsp, \eqsp
a \!= \tfrac{(1-\gamma) \minlambda }{2} \eqsp, \eqsp
\step_{\infty} \!= \tfrac{(1-\gamma)}{4}
\eqsp.
\end{gather}
\end{claim}
We refer to \cite{mangold2024scafflsa} for a proof of this claim.

\paragraph{Application of \perlsa~to TD learning.}
We can now instantiate the results of \Cref{sec:non-asymptotic-rates} in the setting of personalized TD learning. In particular, we derive two sample-complexity guarantees corresponding to the low- and high-heterogeneity regimes, determined by the level of heterogeneity of the local operators $\zeta_{\rm TD}$.
\begin{corollary}[Sample complexity of personalized federated TD:
low-heterogeneity regime]
\label{cor:pf-td-low-sample-complexity-main}
Assume \Cref{assum:generative_model} and \Cref{assum:feature_design}, and suppose that $
\zeta_{\rm TD}^2
< (1-\gamma)^2\nu^2/16,$
and set $\beta=1/N$.
Let $\epsilon>0$ and choose the step size such that
\begin{align}
\textstyle
\eta
\leq
\min\left\{
\frac{1-\gamma}{4},
\frac{(1-\gamma)\nu}{32(1+\gamma)^2},
\frac{(1-\gamma)\nu N\epsilon}
{288\bar\sigma_{\varepsilon,{\rm TD}}^2},
\frac{(1-\gamma)\nu N^2\epsilon}
{192(1+\gamma)^2\|\bar e^0\|_2^2}
\right\}.
\label{eq:pf-td-low-step-complexity-main}
\end{align}
Then \perlsa\, satisfies $
\frac1N\sum_{c=1}^N
\PE\|\theta_c^T-\theta_c^\star\|_2^2
\leq\epsilon$
, provided
\begin{align}
\nonumber
\textstyle
T
\gtrsim
\frac{1}{(1-\gamma)\nu}
\max
& \left\{\textstyle
\frac{1}{1-\gamma},
\frac{(1+\gamma)^2}{(1-\gamma)\nu},
\frac{\bar\sigma_{\varepsilon,{\rm TD}}^2}
{(1-\gamma)\nu N\epsilon},
\frac{
(1+\gamma)^2\|\bar e^0\|_2^2
}{
(1-\gamma)\nu N^2\epsilon
}
\right\}
\\
&\textstyle \qquad\times \textstyle
\max\left\{
\log\left(
\frac{\|\bar e^0\|_2^2}{\epsilon}
\right),
N
\log\left(
(
D^0+
\frac{
\zeta_{\rm TD}^2
}{
(1-\gamma)^2\nu^2N
}
\|\bar e^0\|_2^2
)/
\epsilon
\right)
\right\}.
\label{eq:pf-td-low-sample-complexity}
\end{align}
\end{corollary}
The low-heterogeneity condition takes a particularly interpretable form in TD learning: $
\zeta_{\rm TD}<\frac{(1-\gamma)\nu}{4}.$
Hence, shorter effective horizons (larger $1-\gamma$) and better-conditioned feature representations (larger $\nu$) allow for greater heterogeneity while remaining in the favorable regime. Moreover, the previous result shows that when the initial disagreement is negligible relative to the initial consensus error, \perlsa\, achieves the desired linear-speedup rate
$\widetilde{\mathcal{O}}(1/(N\epsilon))$
while still converging to each agent's personalized TD solution. We now turn to the complementary high-heterogeneity regime.
\begin{corollary}[Sample and communication complexity of personalized federated TD:
high-heterogeneity regime]
\label{cor:pf-td-high-sample-complexity-main}
Assume \Cref{assum:generative_model} and \Cref{assum:feature_design}, and suppose that
$
\zeta_{\rm TD}^2
\geq (1-\gamma)^2\nu^2/16$, and set $
\beta=1/N.$
Let $\epsilon>0$ and choose the step size such that
\begin{align}
\textstyle
\eta
\leq
\min\left\{
\frac{1-\gamma}{4},
\frac{(1-\gamma)\nu}{32(1+\gamma)^2},
\frac{
(1-\gamma)^3\nu^3N\epsilon
}{
1344\zeta_{\rm TD}^2
\bar\sigma_{\varepsilon,{\rm TD}}^2
}
\right\}.
\label{eq:pf-td-high-step-complexity}
\end{align}
Then \perlsa\, satisfies $
\frac1N\sum_{c=1}^N
\PE\|\theta_c^T-\theta_c^\star\|_2^2
\leq\epsilon$, provided
\begin{align}
\nonumber
\textstyle
T
\gtrsim
\frac{1}{(1-\gamma)\nu}
\max\left\{
\frac{1}{1-\gamma},
\frac{(1+\gamma)^2}{(1-\gamma)\nu},
\frac{
\zeta_{\rm TD}^2
\bar\sigma_{\varepsilon,{\rm TD}}^2
}{
(1-\gamma)^3\nu^3N\epsilon
}
\right\}
\max\left\{
\log\left(
\frac{\|\bar e^0\|_2^2}{\epsilon}
\right),
N
\log\left(
\frac{
\zeta_{\rm TD}^2
\left(
D^0+\frac1N\|\bar e^0\|_2^2
\right)
}{
(1-\gamma)^2\nu^2\epsilon
}
\right)
\right\}.
\label{eq:pf-td-high-sample-complexity-main}
\end{align}
\end{corollary}
The high-heterogeneity result illustrates the second side of the best-of-both-worlds behavior. Even when the local TD problems are substantially different, \perlsa\, still estimates each agent's local value function to any desired precision, while recovering the same $\widetilde{\mathcal{O}}(1/\epsilon)$ dependence on the target accuracy as independent TD learning.

\section{Numerical Experiments}
\label{sec:numerical}
\begin{figure}[t]
\centering
\subfigure[$\zeta_{\text{ker}}=0.0,\zeta_{\text{rew}}=0.0$][$\zeta_{\text{ker}}=0$, $\zeta_{\text{rew}}=0$]{\includegraphics[width=0.24\linewidth]{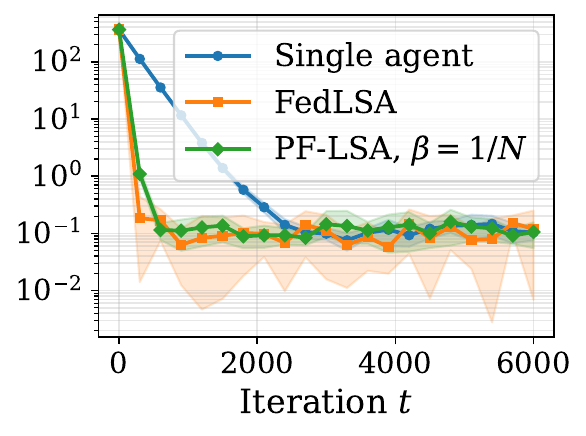}\label{fig:1}}%
\subfigure[$\zeta_{\text{ker}}=0.0,\zeta_{\text{rew}}=0.1$][$\zeta_{\text{ker}}=0$, $\zeta_{\text{rew}}=0.1$]{\includegraphics[width=0.24\linewidth]{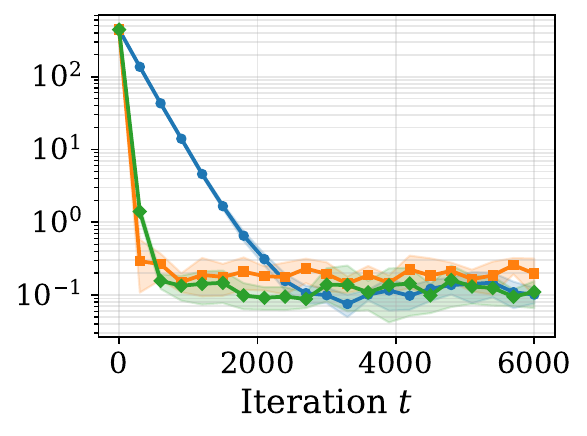}\label{fig:2}}%
\subfigure[$\zeta_{\text{ker}}=0.0,\zeta_{\text{rew}}=0.5$][$\zeta_{\text{ker}}=0$, $\zeta_{\text{rew}}=0.5$]{\includegraphics[width=0.24\linewidth]{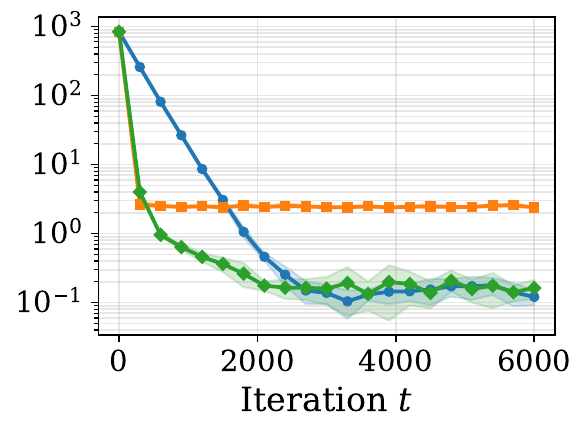}\label{fig:3}}
\subfigure[$\zeta_{\text{ker}}=0,\zeta_{\text{rew}}=1$][$\zeta_{\text{ker}}=0$, $\zeta_{\text{rew}}=1$]{\includegraphics[width=0.24\linewidth]{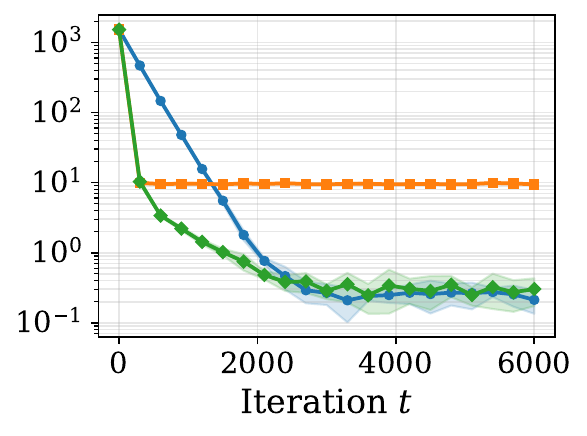}\label{fig:4}}

\subfigure[$\zeta_{\text{ker}}=1,\zeta_{\text{rew}}=0$][$\zeta_{\text{ker}}=1$, $\zeta_{\text{rew}}=0$]{\includegraphics[width=0.24\linewidth]{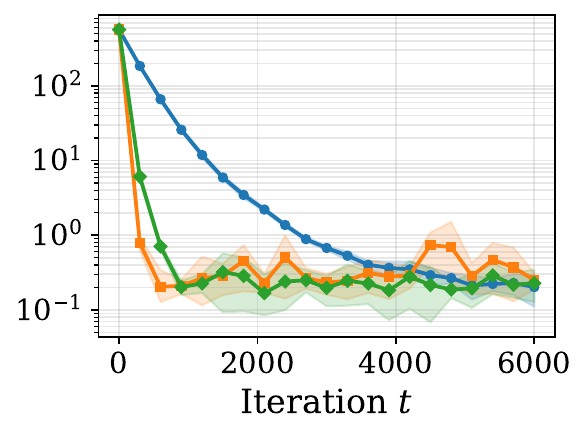}\label{fig:show-bias-1}}%
\subfigure[$\zeta_{\text{ker}}=1,\zeta_{\text{rew}}=0.1$][$\zeta_{\text{ker}}=1$, $\zeta_{\text{rew}}=0.1$]{\includegraphics[width=0.24\linewidth]{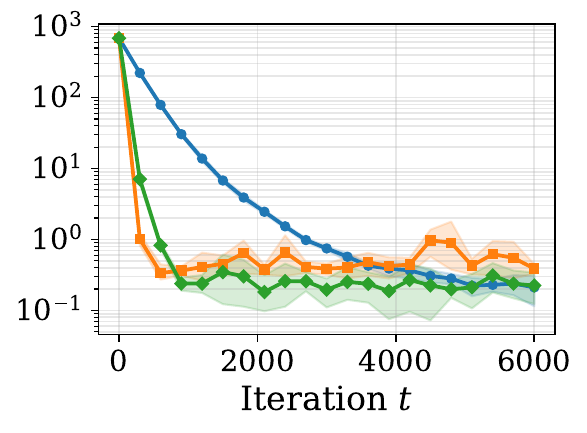}\label{fig:show-bias-2}}
\subfigure[$\zeta_{\text{ker}}=1,\zeta_{\text{rew}}=0.5$][$\zeta_{\text{ker}}=1$, $\zeta_{\text{rew}}=0.5$]{\includegraphics[width=0.24\linewidth]{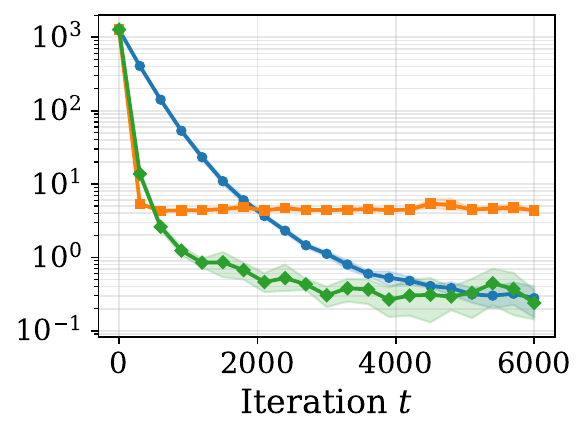}\label{fig:show-bias-3}}%
\subfigure[$\zeta_{\text{ker}}=1,\zeta_{\text{rew}}=1$][$\zeta_{\text{ker}}=1$, $\zeta_{\text{rew}}=1$]{\includegraphics[width=0.24\linewidth]{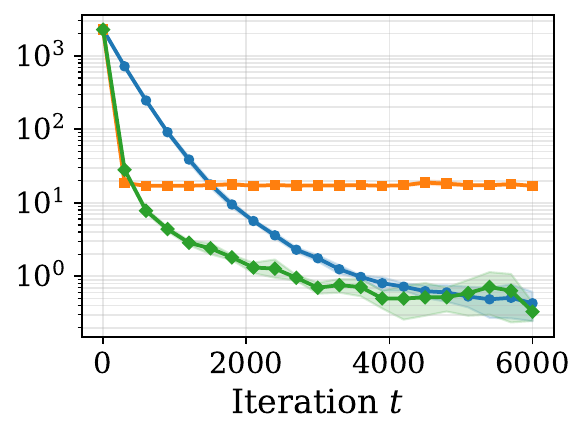}\label{fig:show-bias-4}}

\vspace{-0.1em}

\caption{Average MSE over agents $\frac{1}{\nagent} \sum_{c=1}^{\nagent} \left\| \theta_{c}^{t} - \theta_{c}^{\star} \right\|_{2}^2$ as a function of the number of iterations for \perlsa, \FLSA, and Single agent LSA applied to personalized federated TD(0) under varying levels of heterogeneity of the transition kernels $\zeta_{\text{ker}}$ and rewards $\zeta_{\text{rew}}$, for a fixed number of agents $\nagent=10$. The step-size used for Single agent LSA is $\step = 0.2$, for \perlsa, we use $\step = 0.2\times \nagent/2$, and for \FLSA, we use $\step = 0.2\times \nagent$.  
For each algorithm, we report the average MSE over agents and variance over $5$ runs.}
\label{fig:numerical_experiments}
\vspace{-0.2em}
\end{figure}

In this section, we demonstrate the performance of \perlsa~ under varying levels of heterogeneity.
We consider the Garnet problem \citep{archibald1995generation,geist2014off},
with $n=10$ states embedded in $d=5$ dimensions, $a=10$ actions, and each state is linked to $b=3$ others in the transition kernel.
We aim to estimate the value function of the policy that chooses actions uniformly at random, in homogeneous and heterogeneous setups. The experiments were conducted on a machine equipped with a 13th Gen Intel(R) Core(TM) i9-13900H CPU running at 2.60 GHz.

\textbf{\perlsa~ properly handles heterogeneity.}
Figure~\ref{fig:numerical_experiments} shows that \perlsa~ remains stable and converges towards the personalized solutions across all considered levels of heterogeneity.
When the agents are homogeneous, all methods converge to comparable limiting errors, although \perlsa~ reaches this regime much faster than Single agent LSA.
As the reward heterogeneity $\zeta_{\mathrm{rew}}$ increases, \FLSA~ no longer converges to the personalized solutions and instead plateaus at a large error, reflecting the bias induced by learning a single global model.
This effect is especially visible for $\zeta_{\mathrm{rew}}\in\{0.5,1\}$, where the error of \FLSA~ remains several orders of magnitude above that of \perlsa.
By contrast, \perlsa~ keeps a low final error even in the most heterogeneous settings.

\textbf{\perlsa~ successfully accelerates the learning process.}
Compared with Single agent LSA, \perlsa~ consistently reaches the low-error regime in significantly fewer iterations.
This acceleration is already visible in the homogeneous setting and persists under both reward and transition-kernel heterogeneity.
The improvement comes from the gradient-mixing mechanism, which allows each agent to benefit from the stochastic updates collected by all agents, thereby reducing the effective variance of the update.
This is consistent with our theoretical prediction that, for $\beta=1/N$, the averaged component of the dynamics enjoys a linear speed-up in the number of agents.
Empirically, \perlsa~ typically matches or improves upon the final accuracy of Single agent LSA while requiring substantially fewer iterations to reach the same error level.

\textbf{\perlsa~ achieves the best of both worlds between \FLSA~ and single agent learning.}
The experiments highlight the complementary limitations of the two baselines.
Single agent LSA is fully personalized, but it does not exploit the samples collected by the other agents and therefore converges slowly.
On the other hand, \FLSA~ benefits from collaboration and can converge quickly in nearly homogeneous settings, but it learns a shared solution and becomes biased when the agents are heterogeneous.
\perlsa~ combines the advantages of both approaches: it uses global information through the averaged update, which accelerates learning, while retaining a local update component, which guarantees convergence to each agent's own solution.
The results in Figure~\ref{fig:numerical_experiments} therefore support the main message of the paper: a minimal gradient-mixing mechanism is sufficient to obtain fast collaborative learning without sacrificing personalization.

\section{Conclusion and Perspectives}
\label{sec:conclu}
We introduce \perlsa, a minimalist federated algorithm for personalized LSA, that relies on a lightweight mixing of local updates.
Remarkably, \perlsa~finds agent-specific optimal solutions without imposing any assumptions on the heterogeneity level, as is typical in federated learning. 
Our theoretical analysis establishes that \perlsa~achieves true personalization while still benefiting from collaboration across the network. 
These theoretical results build upon a precise non-asymptotic characterization of the error, which we decompose into "consensus" , that is shared by all agents and "disagreement".
We identify regimes in which this simple mechanism yields a linear speed-up over isolated learning. 
In particular, we show that collaboration is most effective when heterogeneity remains controlled, allowing information sharing to accelerate convergence while still guaranteeing convergence to each agent’s optimal solution.
Our work primarily aims to highlight the remarkable property of a very simple combination of local and global updates.  Extending the present results beyond i.i.d. noise to Markovian settings would be an important step toward applications in reinforcement learning and temporal difference learning.
More broadly, future work could explore algorithmic enhancements such as warm-start strategies to reduce early-stage disagreement and preconditioning methods to mitigate operator heterogeneity.

\section*{Acknowledgements}

The work of S. Labbi, and P. Mangold has been supported by Technology Innovation Institute (TII), project Fed2Learn. The work of E. Moulines has been partly funded by the European Union (ERC-2022-SYG-OCEAN-101071601). Views and opinions expressed are however those of the author(s) only and do not necessarily reflect those of the European Union or the European Research Council Executive Agency. Neither the European Union nor the granting authority can be held responsible for them.

\bibliography{references}

@inproceedings{mou2020linear,
  title={On linear stochastic approximation: Fine-grained {P}olyak-{R}uppert and non-asymptotic concentration},
  author={Mou, Wenlong and Li, Chris Junchi and Wainwright, Martin J and Bartlett, Peter L and Jordan, Michael I},
  booktitle={Conference on Learning Theory},
  pages={2947--2997},
  year={2020},
  organization={PMLR}
}

@inproceedings{khodadadian2022federated,
  title={Federated reinforcement learning: Linear speedup under markovian sampling},
  author={Khodadadian, Sajad and Sharma, Pranay and Joshi, Gauri and Maguluri, Siva Theja},
  booktitle={International Conference on Machine Learning},
  pages={10997--11057},
  year={2022},
  organization={PMLR}
}

@article{li2020federated,
  title={Federated learning: Challenges, methods, and future directions},
  author={Li, Tian and Sahu, Anit Kumar and Talwalkar, Ameet and Smith, Virginia},
  journal={IEEE signal processing magazine},
  volume={37},
  number={3},
  pages={50--60},
  year={2020},
  publisher={IEEE}
}

@inproceedings{mcmahan2017communication,
  title={Communication-efficient learning of deep networks from decentralized data},
  author={McMahan, Brendan and Moore, Eider and Ramage, Daniel and Hampson, Seth and y Arcas, Blaise Aguera},
  booktitle={Artificial intelligence and statistics},
  pages={1273--1282},
  year={2017},
  organization={PMLR}
}

@InProceedings{samsonov2023finite,
  title = 	 {Improved {H}igh-{P}robability {B}ounds for the {T}emporal {D}ifference {L}earning {A}lgorithm via {E}xponential {S}tability},
  author =       {Samsonov, Sergey and Tiapkin, Daniil and Naumov, Alexey and Moulines, Eric},
  booktitle = 	 {Proceedings of Thirty Seventh Conference on Learning Theory},
  pages = 	 {4511--4547},
  year = 	 {2024},
  editor = 	 {Agrawal, Shipra and Roth, Aaron},
  volume = 	 {247},
  series = 	 {Proceedings of Machine Learning Research},
  month = 	 {30 Jun--03 Jul},
  publisher =    {PMLR},
  url = 	 {https://proceedings.mlr.press/v247/samsonov24a.html},
}

@inproceedings{dalal:td0:2018,
	author = {Dalal, G. and Sz{\"o}r{\'e}nyi, Bal{\'a}zs and Thoppe, G. and Mannor, S.},
	booktitle = {Thirty-Second AAAI Conference on Artificial Intelligence},
	title = {{Finite sample analyses for TD(0) with function approximation}},
	year = {2018}}

@inproceedings{bhandari2018finite,
	author = {Bhandari, J. and Russo, D. and Singal, R.},
	booktitle = {Conference On Learning Theory},
	pages = {1691--1692},
	title = {A Finite Time Analysis of Temporal Difference Learning With Linear Function Approximation},
	year = {2018}}

@article{zhang2025personalized,
  title={Personalized Collaborative Learning with Affinity-Based Variance Reduction},
  author={Zhang, Chenyu and Azizan, Navid},
  journal={arXiv preprint arXiv:2510.16232},
  year={2025}
}

@article{philippenko2025adaptive,
  title={Adaptive collaboration for online personalized distributed learning with heterogeneous clients},
  author={Philippenko, Constantin and Bars, Batiste Le and Scaman, Kevin and Massouli{\~A}{\v{S}}, Laurent},
  journal={arXiv preprint arXiv:2507.06844},
  year={2025}
}

@article{kairouz2021advances,
  title={Advances and open problems in federated learning},
  author={Kairouz, Peter and McMahan, H Brendan and Avent, Brendan and Bellet, Aur{\'e}lien and Bennis, Mehdi and Bhagoji, Arjun Nitin and Bonawitz, Kallista and Charles, Zachary and Cormode, Graham and Cummings, Rachel and others},
  journal={Foundations and Trends{\textregistered} in Machine Learning},
  volume={14},
  number={1--2},
  pages={1--210},
  year={2021},
  publisher={Now Publishers, Inc.}
}

@inproceedings{patil2023finite,
  title={Finite time analysis of temporal difference learning with linear function approximation: Tail averaging and regularisation},
  author={Patil, Gandharv and Prashanth, LA and Nagaraj, Dheeraj and Precup, Doina},
  booktitle={International Conference on Artificial Intelligence and Statistics},
  pages={5438--5448},
  year={2023},
  organization={PMLR}
}

@article{li2023sharp,
  author={Li, Gen and Wu, Weichen and Chi, Yuejie and Ma, Cong and Rinaldo, Alessandro and Wei, Yuting},
  journal={IEEE Transactions on Information Theory}, 
  title={High-Probability Sample Complexities for Policy Evaluation With Linear Function Approximation}, 
  year={2024},
  volume={70},
  number={8},
  pages={5969-5999},
  doi={10.1109/TIT.2024.3394685}
}

@article{tsitsiklis:td:1997,
	author = {J. N. {Tsitsiklis} and B. {Van Roy}},
	doi = {10.1109/9.580874},
	issn = {2334-3303},
	journal = {IEEE Transactions on Automatic Control},
	month = {May},
	number = {5},
	pages = {674-690},
	title = {An analysis of temporal-difference learning with function approximation},
	volume = {42},
	year = {1997}}

@article{zhao2018federated,
  title={Federated learning with non-iid data},
  author={Zhao, Yue and Li, Meng and Lai, Liangzhen and Suda, Naveen and Civin, Damon and Chandra, Vikas},
  journal={arXiv preprint arXiv:1806.00582},
  year={2018}
}

@article{wang2023federated,
  title={Federated temporal difference learning with linear function approximation under environmental heterogeneity},
  author={Wang, Han and Mitra, Aritra and Hassani, Hamed and Pappas, George J and Anderson, James},
  journal={arXiv preprint arXiv:2302.02212},
  year={2023}
}

@article{geist2014off,
  title={Off-policy learning with eligibility traces: a survey.},
  author={Geist, Matthieu and Scherrer, Bruno and others},
  journal={J. Mach. Learn. Res.},
  volume={15},
  number={1},
  pages={289--333},
  year={2014}
}

@article{archibald1995generation,
  title={On the generation of markov decision processes},
  author={Archibald, TW and McKinnon, KIM and Thomas, LC},
  journal={Journal of the Operational Research Society},
  volume={46},
  number={3},
  pages={354--361},
  year={1995},
  publisher={Taylor \& Francis}
}

@inproceedings{doan2019FiniteTime,
	title = {Finite-{Time} {Analysis} of {Distributed} {TD}(0) with {Linear} {Function} {Approximation} on {Multi}-{Agent} {Reinforcement} {Learning}},
	url = {https://proceedings.mlr.press/v97/doan19a.html},
	language = {en},
	urldate = {2024-01-31},
	booktitle = {Proceedings of the 36th {International} {Conference} on {Machine} {Learning}},
	publisher = {PMLR},
	author = {Doan, Thinh and Maguluri, Siva and Romberg, Justin},
	month = may,
	year = {2019},
	note = {ISSN: 2640-3498},
	pages = {1626--1635},
}

@article{robbins1951stochastic,
  title={A stochastic approximation method},
  author={Robbins, Herbert and Monro, Sutton},
  journal={The annals of mathematical statistics},
  pages={400--407},
  year={1951},
  publisher={JSTOR}
}

@book{benveniste1990adaptive,
  title={Adaptive Algorithms and Stochastic Approximations},
  author={Benveniste, A. and M{\'e}tivier, M. and Priouret, P.},
  isbn={9783540528944},
  lccn={91127124},
  series={Applications of mathematics},
  url={https://books.google.fr/books?id=vtuovgEACAAJ},
  year={1990},
  publisher={Springer-Verlag}
}

@inproceedings{defossez2015averaged,
  title={Averaged least-mean-squares: Bias-variance trade-offs and optimal sampling distributions},
  author={D{\'e}fossez, Alexandre and Bach, Francis},
  booktitle={Artificial Intelligence and Statistics},
  pages={205--213},
  year={2015},
  organization={PMLR}
}

@article{tsitsiklis1996analysis,
  title={Analysis of temporal-diffference learning with function approximation},
  author={Tsitsiklis, John and Van Roy, Benjamin},
  journal={Advances in neural information processing systems},
  volume={9},
  year={1996}
}

@book{sutton2018reinforcement,
  title={Reinforcement Learning: An Introduction},
  author={Sutton, Richard S and Barto, Andrew G},
  year={2018},
  publisher={MIT Press}
}

@book{borkar2008stochastic,
  title={Stochastic approximation: a dynamical systems viewpoint},
  author={Borkar, Vivek S and Borkar, Vivek S},
  volume={100},
  year={2008},
  publisher={Springer}
}

@article{smith2017federated,
  title={Federated multi-task learning},
  author={Smith, Virginia and Chiang, Chao-Kai and Sanjabi, Maziar and Talwalkar, Ameet S},
  journal={Advances in neural information processing systems},
  volume={30},
  year={2017}
}

@article{sattler2020clustered,
  title={Clustered federated learning: Model-agnostic distributed multitask optimization under privacy constraints},
  author={Sattler, Felix and M{\"u}ller, Klaus-Robert and Samek, Wojciech},
  journal={IEEE transactions on neural networks and learning systems},
  volume={32},
  number={8},
  pages={3710--3722},
  year={2020},
  publisher={IEEE}
}

@article{mansour2020three,
  title={Three approaches for personalization with applications to federated learning},
  author={Mansour, Yishay and Mohri, Mehryar and Ro, Jae and Suresh, Ananda Theertha},
  journal={arXiv preprint arXiv:2002.10619},
  year={2020}
}

@article{t2020personalized,
  title={Personalized federated learning with moreau envelopes},
  author={T Dinh, Canh and Tran, Nguyen and Nguyen, Josh},
  journal={Advances in neural information processing systems},
  volume={33},
  pages={21394--21405},
  year={2020}
}

@article{
wang2024federated,
title={Federated {TD} Learning with Linear Function Approximation under Environmental Heterogeneity},
author={Han Wang and Aritra Mitra and Hamed Hassani and George J. Pappas and James Anderson},
journal={Transactions on Machine Learning Research},
issn={2835-8856},
year={2024},
url={https://openreview.net/forum?id=hdQspgyFrk},
note={}
}

@article{mangold2024scafflsa,
  title   = {SCAFFLSA: Taming Heterogeneity in Federated Linear Stochastic Approximation and TD Learning},
  author  = {Mangold, Paul and Samsonov, Sergey and Labbi, Safwan and Levin, Ilya and Alami, Reda and Naumov, Alexey and Moulines, Eric},
  journal = {Advances in Neural Information Processing Systems (NeurIPS)},
  volume  = {37},
  year    = {2024}
}

@InProceedings{pmlr-v300-labbi26a,
  title = 	 { On Global Convergence Rates for Federated Softmax Policy Gradient Under HeterogeneousEnvironments },
  author =       {Labbi, Safwan and Mangold, Paul and Tiapkin, Daniil and Moulines, Eric},
  booktitle = 	 {Proceedings of The 29th International Conference on Artificial Intelligence and Statistics},
  pages = 	 {4960--4968},
  year = 	 {2026},
  editor = 	 {Khan, Emtiyaz and Li, Yingzhen and Solin, Arno and Ramdas, Aaditya},
  volume = 	 {300},
  series = 	 {Proceedings of Machine Learning Research},
  month = 	 {02--05 May},
  publisher =    {PMLR},
  url = 	 {https://proceedings.mlr.press/v300/labbi26a.html}
}

@InProceedings{pmlr-v258-labbi25a,
  title = 	 {Federated UCBVI: Communication-Efficient Federated Regret Minimization with Heterogeneous Agents},
  author =       {Labbi, Safwan and Tiapkin, Daniil and Mancini, Lorenzo and Mangold, Paul and Moulines, Eric},
  booktitle = 	 {Proceedings of The 28th International Conference on Artificial Intelligence and Statistics},
  pages = 	 {1315--1323},
  year = 	 {2025},
  editor = 	 {Li, Yingzhen and Mandt, Stephan and Agrawal, Shipra and Khan, Emtiyaz},
  volume = 	 {258},
  series = 	 {Proceedings of Machine Learning Research},
  month = 	 {03--05 May},
  publisher =    {PMLR},
  url = 	 {https://proceedings.mlr.press/v258/labbi25a.html}
}

@article{mangold2025convergence,
  title={Convergence guarantees for federated sarsa with local training and heterogeneous agents},
  author={Mangold, Paul and Berthier, Elo{\~A}{\NG}se and Moulines, Eric},
  journal={arXiv preprint arXiv:2512.17688},
  year={2025}
}
\bibliographystyle{iclr2027_conference}

\clearpage
\newpage

\appendix

\section{Convergence of a Lyapunov function}
\label{sec:appendix_lyapunov}
\allowdisplaybreaks
In this appendix, we denote the errors and disagreement as follows:
\begin{align}
\label{def:errors_and_disagreement}
\projerror{c}{t}\eqdef \theta_{c}^t - \theta_c^{\star} \eqsp,\quad
\avgprojerror{t}\eqdef\frac{1}{\nagent}\sum_{c=1}^\nagent \projerror{c}{t} \eqsp,\quad
\disagreement{c}{t} \eqdef \projerror{c}{t}-\avgprojerror{t} \eqsp,\quad
\normdisagreement{t}\eqdef \frac{1}{\nagent}\sum_{c=1}^{\nagent} \|\disagreement{c}{t}\|^2 \eqsp.
\end{align}
We also introduce the short-hand notations for the noise variables
\begin{align}
\label{def:short-hand-noise}
\nfuncnoise{t+1}[c] 
=
\nfuncnoise{}[c](\State_{c}^{t+1}) 
\eqsp,
\quad
\avgfuncnoise{t+1} 
=
\frac{1}{\nagent}\sum_{c=1}^\nagent
\nfuncnoise{{t+1}}[c] 
\eqsp,
\quad
\ndfuncnoise{t+1}[c] 
=
\avgfuncnoise{t+1}
- \nfuncnoise{t+1}[c]
\eqsp,
\end{align}
as well as
\begin{align}
\label{def:short-hand-noise-Ab}
\shnfuncA[c]{t+1}
= \shnfuncA[c]{}(\State_{c}^{t+1}) 
\eqsp,
\quad
\shnfuncb[c]{t+1}
= \nfuncb[c]{\State_{c}^{t+1}} 
\eqsp,
\end{align}
and define the filtration $\{\filtr^{t}\}_{t\ge 0}$ that contains all randomness up to time $t$.
\begin{align}
\nonumber
\filtr^{t}
\eqdef 
\sigma\Big( \State_{c}^p : p \in \{1, \dots, t \}, \eqsp c \in [\nagent]\Big) \eqsp, 
\end{align}
The following lemma explicits the recursion on the individual error and the disagreement.

\begin{theorem}
\label{lem:unified-weighted-energy}
Assume \Cref{assum:exp_stability} and \Cref{assum:noise-level-flsa}.  Let
$0<\mixingparam\leq1$, let
\[
\step\in\left(0,\min\left\{\step_\infty,
\frac{a}{16\bConst{A}^2}\right\}\right),
\]
and let $0<\lambda\leq1$ satisfy
\begin{align}
(1-\lambda)\hgty
\leq \frac{a\sqrt{\lambda}}2 \eqsp, \quad
\frac{\step\bConst{A}^2}{a}
\frac{\max\{\mixingparam,\lambda\}}{\lambda}
\leq\frac16 \eqsp.
\label{eq:unified-lambda-conditions}
\end{align}
Then, for every $t\geq0$,
\begin{align}
\PE\left[
\normdisagreement{t+1}
+\mixingparam\lambda\|\avgprojerror{t+1}\|_2^2
\right] \leq
(1-\step\mixingparam a)
\PE\left[
\normdisagreement{t}
+\mixingparam\lambda\|\avgprojerror{t}\|_2^2
\right] +
2\step^2
\left[
\mixingparam^2\left(1-\frac1{\nagent}\right)
+\frac{\mixingparam\lambda}{\nagent}
\right]\avgnoise^2 \eqsp.
\label{eq:unified-weighted-energy-recursion}
\end{align}
\end{theorem}
\begin{proof}
The average-error and disagreement recursions are
\begin{align}
\avgprojerror{t+1}
&=
\avgprojerror{t}
-\frac{\step}{\nagent}\sum_{c=1}^{\nagent}
\left(
\shnfuncA[c]{t+1}e_c^t
+\nfuncnoise{t+1}[c]
\right),
\label{eq:unified-direct-average-recursion}
\\
d_c^{t+1}
&=
d_c^t
-\step\mixingparam
\Bigg[
\shnfuncA[c]{t+1}e_c^t
+\nfuncnoise{t+1}[c]
-\frac1{\nagent}\sum_{c'=1}^{\nagent}
\left(
\shnfuncA[c']{t+1}e_{c'}^t
+\nfuncnoise{t+1}[c']
\right)
\Bigg].
\label{eq:unified-direct-disagreement-recursion}
\end{align}
Squaring \eqref{eq:unified-direct-disagreement-recursion}, averaging over
$c\in[\nagent]$, and using $\sum_{c=1}^{\nagent}d_c^t=0$ gives
\begin{align}
\nonumber
\normdisagreement{t+1}
&=
\normdisagreement{t}
-\frac{2\step\mixingparam}{\nagent}
\sum_{c=1}^{\nagent}
\left\langle
d_c^t,
\shnfuncA[c]{t+1}e_c^t
+\nfuncnoise{t+1}[c]
\right\rangle
\\
&+\frac{\step^2\mixingparam^2}{\nagent}
\sum_{c=1}^{\nagent}
\Bigg\|
\shnfuncA[c]{t+1}e_c^t
+\nfuncnoise{t+1}[c]
-\frac1{\nagent}\sum_{c'=1}^{\nagent}
\left(
\shnfuncA[c']{t+1}e_{c'}^t
+\nfuncnoise{t+1}[c']
\right)
\Bigg\|_2^2.
\label{eq:unified-disagreement-expansion}
\end{align}
Indeed, the term involving the empirical average in the linear part
vanishes, since
\begin{align}
\sum_{c=1}^{\nagent}
\left\langle
d_c^t,
\frac1{\nagent}\sum_{c'=1}^{\nagent}
\left(
\shnfuncA[c']{t+1}e_{c'}^t
+\nfuncnoise{t+1}[c']
\right)
\right\rangle 
=
\left\langle
\sum_{c=1}^{\nagent} d_c^t,
\frac1{\nagent}\sum_{c'=1}^{\nagent}
\left(
\shnfuncA[c']{t+1}e_{c'}^t
+\nfuncnoise{t+1}[c']
\right)
\right\rangle
=0.
\end{align}
Likewise,
\begin{align}
\nonumber
\mixingparam\lambda\|\avgprojerror{t+1}\|_2^2
&=
\mixingparam\lambda\|\avgprojerror{t}\|_2^2
-2\step\mixingparam\lambda
\left\langle
\avgprojerror{t},
\frac1{\nagent}\sum_{c=1}^{\nagent}
\left(
\shnfuncA[c]{t+1}e_c^t
+\nfuncnoise{t+1}[c]
\right)
\right\rangle
\\
&\quad
+\step^2\mixingparam\lambda
\left\|
\frac1{\nagent}\sum_{c=1}^{\nagent}
\left(
\shnfuncA[c]{t+1}e_c^t
+\nfuncnoise{t+1}[c]
\right)
\right\|_2^2.
\label{eq:unified-average-expansion-first}
\end{align}

By \Cref{assum:noise-level-flsa},
\begin{align}
\PE\left[
\shnfuncA[c]{t+1}e_c^t
+\nfuncnoise{t+1}[c]
\,\middle|\,\filtr^t
\right]
=
\bA[c]e_c^t.
\end{align}
Therefore, taking conditional expectations in
\eqref{eq:unified-disagreement-expansion} and
\eqref{eq:unified-average-expansion-first}, we obtain
\begin{align}
\nonumber
\PE\left[
\normdisagreement{t+1}
+\mixingparam\lambda\|\avgprojerror{t+1}\|_2^2
\,\middle|\,\filtr^t
\right]
&=
\normdisagreement{t}
+\mixingparam\lambda\|\avgprojerror{t}\|_2^2
\\
\nonumber
&\quad
\underbrace{
-2\step\mixingparam
\left[
\frac1{\nagent}\sum_{c=1}^{\nagent}
\left\langle
d_c^t,\bA[c]e_c^t
\right\rangle
+
\lambda
\left\langle
\avgprojerror{t},
\frac1{\nagent}\sum_{c=1}^{\nagent}
\bA[c]e_c^t
\right\rangle
\right]
}_{\term{A}}
\\
\nonumber
&\quad
+\underbrace{
\step^2\PE\Bigg[
\frac{\mixingparam^2}{\nagent}\sum_{c=1}^{\nagent}
\Bigg\|
\shnfuncA[c]{t+1}e_c^t
+\nfuncnoise{t+1}[c]
-\frac1{\nagent}\sum_{c'=1}^{\nagent}
\left(
\shnfuncA[c']{t+1}e_{c'}^t
+\nfuncnoise{t+1}[c']
\right)
\Bigg\|_2^2
\,\Bigg|\,\filtr^t
\Bigg]
}_{\term{B}}
\\
&\quad
+\underbrace{
\step^2\mixingparam\lambda
\PE\Bigg[
\left\|
\frac1{\nagent}\sum_{c=1}^{\nagent}
\left(
\shnfuncA[c]{t+1}e_c^t
+\nfuncnoise{t+1}[c]
\right)
\right\|_2^2
\,\Bigg|\,\filtr^t
\Bigg]
}_{\term{B}}.
\label{eq:unified-weighted-exact-expansion}
\end{align}
\paragraph{Bounding (A).}
By definition of the disagreement,
$d_c^t=e_c^t-\bar e^t$, and hence
$\bar e^t=e_c^t-d_c^t$. Therefore, by linearity of the inner
product,
\begin{align}
\nonumber
\left\langle
\bar e^t,
\frac1{\nagent}\sum_{c=1}^{\nagent}\bA[c]e_c^t
\right\rangle
&=
\frac1{\nagent}\sum_{c=1}^{\nagent}
\left\langle
e_c^t-d_c^t,\bA[c]e_c^t
\right\rangle
\\
&=
\frac1{\nagent}\sum_{c=1}^{\nagent}
\left\langle
e_c^t,\bA[c]e_c^t
\right\rangle
-
\frac1{\nagent}\sum_{c=1}^{\nagent}
\left\langle
d_c^t,\bA[c]e_c^t
\right\rangle.
\end{align}
Substituting this identity into $\term{A}$ and collecting the
coefficients of
$\nagent^{-1}\sum_c\langle d_c^t,\bA[c]e_c^t\rangle$ yields
\begin{align}
\term{A}
=
-2\step\mixingparam
\left[
\lambda \underbrace{\frac1{\nagent}\sum_{c=1}^{\nagent}
\left\langle
e_c^t,\bA[c]e_c^t
\right\rangle}_{\term{A1}}
+
(1-\lambda)\underbrace{\frac1{\nagent}\sum_{c=1}^{\nagent}
\left\langle
d_c^t,\bA[c]e_c^t
\right\rangle}_{\term{A2}}
\right].
\end{align}
Using Lemma \ref{lem:bound_scalar_product}, we obtain the following bound on $\term{A1}$
\begin{align}
\term{A1} = \frac1{\nagent}\sum_{c=1}^{\nagent}
\left\langle
e_c^t,\bA[c]e_c^t
\right\rangle
\geq
\frac a{\nagent}\sum_{c=1}^{\nagent}\|e_c^t\|_2^2 =
a\left(
\|\avgprojerror{t}\|_2^2
+\normdisagreement{t}
\right).
\label{eq:unified-error-monotonicity_a1}
\end{align}
Since $\sum_c\disagreement{c}{t}=0$ and
$\projerror{c}{t}=\avgprojerror{t}+\disagreement{c}{t}$, we can decompose $\term{A2}$ as
\begin{align}
\term{A2} = \frac1{\nagent}\sum_{c=1}^{\nagent}
\left\langle
\disagreement{c}{t},
\bA[c]\projerror{c}{t}
\right\rangle =
\frac1{\nagent}\sum_{c=1}^{\nagent}
\left\langle
\disagreement{c}{t},
\bA[c]\disagreement{c}{t}
\right\rangle
+\frac1{\nagent}\sum_{c=1}^{\nagent}
\left\langle
\disagreement{c}{t},
(\bA[c]-\bA)\avgprojerror{t}
\right\rangle.
\label{eq:unified-disagreement-drift}
\end{align}
Using Lemma \ref{lem:bound_scalar_product} and Cauchy-Schwartz inequality, we obtain the following bound on $\term{A2}$
\begin{align}
\nonumber
\term{A2} = \frac1{\nagent}\sum_{c=1}^{\nagent}
\left\langle
d_c^t,\bA[c]e_c^t
\right\rangle
&\geq
a\normdisagreement{t}
-\frac1{\nagent}\sum_{c=1}^{\nagent}
\|d_c^t\|_2
\left\|
\bA[c]
-\frac1{\nagent}\sum_{c'=1}^{\nagent}\bA[c']
\right\|
\|\avgprojerror{t}\|_2
\\
& = 
a\normdisagreement{t}
-\hgty
\|\avgprojerror{t}\|_2
\sqrt{\normdisagreement{t}},
\label{eq:unified-disagreement-lower-bound}
\end{align}
where the last equality follows from the definition of $\hgty$, see \eqref{def:heterogeneity_agents}.
Combining \eqref{eq:unified-error-monotonicity_a1} and
\eqref{eq:unified-disagreement-lower-bound}, the expression inside the
square brackets in $\term{A}$ is bounded from below by
\begin{align}
a\lambda\|\avgprojerror{t}\|_2^2
+a\normdisagreement{t}
-(1-\lambda)\hgty
\|\avgprojerror{t}\|_2
\sqrt{\normdisagreement{t}} \eqsp.
\label{eq:unified-linear-lower-bound}
\end{align}
By the first condition in
\eqref{eq:unified-lambda-conditions},
\begin{align}
\nonumber
(1-\lambda)\hgty
\|\avgprojerror{t}\|_2
\sqrt{\normdisagreement{t}}
\leq
\frac{a\sqrt{\lambda}}2
\|\avgprojerror{t}\|_2
\sqrt{\normdisagreement{t}}
\leq
\frac a4
\left(
\lambda\|\avgprojerror{t}\|_2^2
+\normdisagreement{t}
\right) \eqsp,
\end{align}
where the second inequality follows from Young's inequality. Consequently,
\begin{align}
a\lambda\|\avgprojerror{t}\|_2^2
+a\normdisagreement{t}
-(1-\lambda)\hgty
\|\avgprojerror{t}\|_2
\sqrt{\normdisagreement{t}}
\geq
\frac{3a}{4}
\left(
\lambda\|\avgprojerror{t}\|_2^2
+\normdisagreement{t}
\right),
\label{eq:unified-linear-three-fourths}
\end{align}
and therefore
\begin{align}
\term{A}
\leq
-\frac{3}{2}\step\mixingparam a
\left(
\lambda\|\avgprojerror{t}\|_2^2
+\normdisagreement{t}
\right).
\label{eq:unified-A-bound}
\end{align}

\paragraph{Bounding (B).}
Applying the variance decomposition around the conditional mean yields
\begin{align}
\term{B} &= 
\step^2\PE\Bigg[
\frac{\mixingparam^2}{\nagent}\sum_{c=1}^{\nagent}
\Bigg\|
\shnfuncA[c]{t+1}e_c^t
+\nfuncnoise{t+1}[c]
-\frac1{\nagent}\sum_{c'=1}^{\nagent}
\left(
\shnfuncA[c']{t+1}e_{c'}^t
+\nfuncnoise{t+1}[c']
\right)
\Bigg\|_2^2
\,\Bigg|\,\filtr^t
\Bigg]
\\
&
+
\step^2\mixingparam\lambda
\PE\Bigg[
\left\|
\frac1{\nagent}\sum_{c=1}^{\nagent}
\left(
\shnfuncA[c]{t+1}e_c^t
+\nfuncnoise{t+1}[c]
\right)
\right\|_2^2
\,\Bigg|\,\filtr^t
\Bigg] 
\\
& = \underbrace{\step^2\PE\Bigg[
\frac{\mixingparam^2}{\nagent}\sum_{c=1}^{\nagent}
\Bigg\|
\bA[c]e_c^t
-\frac1{\nagent}\sum_{c'=1}^{\nagent}
\bA[c']e_{c'}^t
\Bigg\|_2^2
\,\Bigg|\,\filtr^t
\Bigg] + \step^2\mixingparam\lambda
\PE\Bigg[
\left\|
\frac1{\nagent}\sum_{c=1}^{\nagent}
\bA[c]e_c^t
\right\|_2^2
\,\Bigg|\,\filtr^t
\Bigg]}_{\term{B1}}
\\
& + \underbrace{\step^2\PE\Bigg[
\frac{\mixingparam^2}{\nagent}\sum_{c=1}^{\nagent}
\Bigg\|
(\shnfuncA[c]{t+1}  -\bA[c] )e_c^t
+\nfuncnoise{t+1}[c]
-\frac1{\nagent}\sum_{c'=1}^{\nagent}
\left(
(\shnfuncA[c']{t+1}- \bA[c']) e_{c'}^t
+\nfuncnoise{t+1}[c']
\right)
\Bigg\|_2^2
\,\Bigg|\,\filtr^t
\Bigg]}_{\term{B2}}
\\
&
+ \underbrace{\step^2\mixingparam\lambda
\PE\Bigg[
\left\|
\frac1{\nagent}\sum_{c=1}^{\nagent}
\left(
(\shnfuncA[c]{t+1} - \bA[c] )e_c^t
+\nfuncnoise{t+1}[c]
\right)
\right\|_2^2
\,\Bigg|\,\filtr^t
\Bigg]}_{\term{B3}} \eqsp.
\end{align}

\textit{- Bounding (B1).} Using the
identity
\begin{align}
\label{eq:pythagoras}
\frac1{\nagent}\sum_{c=1}^{\nagent}
\left\|
x_c-\frac1{\nagent}\sum_{c'=1}^{\nagent}x_{c'}
\right\|_2^2
=
\frac1{\nagent}\sum_{c=1}^{\nagent}\|x_c\|_2^2
-
\left\|
\frac1{\nagent}\sum_{c=1}^{\nagent}x_c
\right\|_2^2,
\end{align}
with $x_c=\bA[c]e_c^t$, we obtain
\begin{align}
\nonumber
&
\frac{\mixingparam^2}{\nagent}\sum_{c=1}^{\nagent}
\left\|
\bA[c]e_c^t
-\frac1{\nagent}\sum_{c'=1}^{\nagent}
\bA[c']e_{c'}^t
\right\|_2^2
+
\mixingparam\lambda
\left\|
\frac1{\nagent}\sum_{c=1}^{\nagent}
\bA[c]e_c^t
\right\|_2^2
\\
&=
\frac{\mixingparam^2}{\nagent}\sum_{c=1}^{\nagent}
\|\bA[c]e_c^t\|_2^2
+
(\mixingparam\lambda-\mixingparam^2)
\left\|
\frac1{\nagent}\sum_{c=1}^{\nagent}
\bA[c]e_c^t
\right\|_2^2 \eqsp.
\label{eq:unified-deterministic-exact}
\end{align}
If $\lambda\leq\mixingparam$, the second term on the right-hand side is
non-positive, and therefore
\begin{align}
\frac{\mixingparam^2}{\nagent}\sum_{c=1}^{\nagent}
\|\bA[c]e_c^t\|_2^2
+
(\mixingparam\lambda-\mixingparam^2)
\left\|
\frac1{\nagent}\sum_{c=1}^{\nagent}
\bA[c]e_c^t
\right\|_2^2
\leq
\frac{\mixingparam^2}{\nagent}
\sum_{c=1}^{\nagent}
\|\bA[c]e_c^t\|_2^2.
\end{align}
If $\lambda>\mixingparam$, Jensen's inequality gives
\begin{align}
\left\|
\frac1{\nagent}\sum_{c=1}^{\nagent}\bA[c]e_c^t
\right\|_2^2
\leq
\frac1{\nagent}\sum_{c=1}^{\nagent}
\|\bA[c]e_c^t\|_2^2,
\end{align}
and hence
\begin{align}
\frac{\mixingparam^2}{\nagent}\sum_{c=1}^{\nagent}
\|\bA[c]e_c^t\|_2^2
+
(\mixingparam\lambda-\mixingparam^2)
\left\|
\frac1{\nagent}\sum_{c=1}^{\nagent}
\bA[c]e_c^t
\right\|_2^2
\leq
\frac{\mixingparam\lambda}{\nagent}
\sum_{c=1}^{\nagent}
\|\bA[c]e_c^t\|_2^2.
\end{align}
Thus, in both cases,
\begin{align}
\term{B1} \leq
\step^2 \mixingparam\max\{\mixingparam,\lambda\}
\bConst{A}^2
\left(
\|\avgprojerror{t}\|_2^2
+\normdisagreement{t}
\right).
\label{eq:unified-deterministic-quadratic}
\end{align}

\textit{- Bounding (B1) + (B3).}
Applying \eqref{eq:pythagoras} with $x_c = (\shnfuncA[c]{t+1} - \bA[c] )e_c^t
+\nfuncnoise{t+1}[c] $ gives
\begin{align}
&\term{B2} + \term{B3} \\
&  \quad =  \step^2\PE\Bigg[
\frac{\mixingparam^2}{\nagent}\sum_{c=1}^{\nagent}
\Bigg\|
(\shnfuncA[c]{t+1}  -\bA[c] )e_c^t
+\nfuncnoise{t+1}[c]
-\frac1{\nagent}\sum_{c'=1}^{\nagent}
\left(
(\shnfuncA[c']{t+1}- \bA[c']) e_{c'}^t
+\nfuncnoise{t+1}[c']
\right)
\Bigg\|_2^2
\,\Bigg|\,\filtr^t
\Bigg]
\\
&
 \quad + \step^2\mixingparam\lambda
\PE\Bigg[
\left\|
\frac1{\nagent}\sum_{c=1}^{\nagent}
\left(
(\shnfuncA[c]{t+1} - \bA[c] )e_c^t
+\nfuncnoise{t+1}[c]
\right)
\right\|_2^2
\,\Bigg|\,\filtr^t
\Bigg] 
\\
&  \quad= \frac{\step^2 \mixingparam^2}{\nagent} \PE\Bigg[ \sum_{c=1}^{\nagent} \|(\shnfuncA[c]{t+1} - \bA[c] )e_c^t
+\nfuncnoise{t+1}[c]\|_2^2\Bigg|\,\filtr^t
\Bigg] 
\\
& \quad +
`\step^2 (\mixingparam\lambda-\mixingparam^2)
\PE\Bigg[ \left\|
\frac1{\nagent}\sum_{c=1}^{\nagent}
(\shnfuncA[c]{t+1} - \bA[c] )e_c^t
+\nfuncnoise{t+1}[c]
\right\|_2^2\Bigg|\,\filtr^t
\Bigg]  
\end{align}
By independence of the noise across
agents, we obtain
\begin{align}
\label{eq:intermediate_bound_b_23}
\term{B2} + \term{B3} &= \step^2 \left[\mixingparam^2 + \frac{\mixingparam\lambda-\mixingparam^2}{\nagent}\right]\frac{1}{\nagent} \PE\Bigg[ \sum_{c=1}^{\nagent} \|(\shnfuncA[c]{t+1} - \bA[c] )e_c^t
+\nfuncnoise{t+1}[c]\|_2^2\Bigg|\,\filtr^t
\Bigg] \eqsp.
\end{align}
Moreover, by Young's inequality and 
\Cref{assum:noise-level-flsa}, we have
\begin{align}
\nonumber
\frac1{\nagent}\sum_{c=1}^{\nagent}
\PE\left[
\left\|
(\shnfuncA[c]{t+1}-\bA[c])e_c^t
+\nfuncnoise{t+1}[c]
\right\|_2^2
\,\middle|\,\filtr^t
\right]
&\leq
\frac{2}{\nagent}\sum_{c=1}^{\nagent}
\PE\left[
\left\|
(\shnfuncA[c]{t+1}-\bA[c])e_c^t
\right\|_2^2
\,\middle|\,\filtr^t
\right]
+
2\avgnoise^2
\\
&\leq
\frac{2\bConst{A}^2}{\nagent}
\sum_{c=1}^{\nagent}\|e_c^t\|_2^2
+
2\avgnoise^2
\\
&=
2\bConst{A}^2
\left(
\|\avgprojerror{t}\|_2^2
+\normdisagreement{t}
\right)
+
2\avgnoise^2.
\label{eq:unified-innovation-variance}
\end{align}
where in the last equality we used $
\frac1{\nagent}\sum_{c=1}^{\nagent}\|e_c^t\|_2^2
=
\|\avgprojerror{t}\|_2^2+\normdisagreement{t}$. Plugging in the previous bound in \eqref{eq:intermediate_bound_b_23} yields
\begin{align}
\term{B2} + \term{B3} &\leq \step^2 \left[\mixingparam^2 + \frac{\mixingparam\lambda-\mixingparam^2}{\nagent}\right] \left(2\bConst{A}^2
\left(
\|\avgprojerror{t}\|_2^2
+\normdisagreement{t}
\right)
+
2\avgnoise^2 \right)
\end{align}
Furthermore, as
\begin{align}
\nonumber
\left[\mixingparam^2 + \frac{\mixingparam\lambda-\mixingparam^2}{\nagent}\right] 
=
\mixingparam
\left[
\left(1-\frac1{\nagent}\right)\mixingparam
+\frac{\lambda}{\nagent}
\right]
\leq
\mixingparam\max\{\mixingparam,\lambda\} \eqsp,
\end{align}
we get
\begin{align}
\nonumber
\term{B2} + \term{B3} \leq
2\step^2
\mixingparam\max\{\mixingparam,\lambda\}
\bConst{A}^2
\left(
\|\avgprojerror{t}\|_2^2
+\normdisagreement{t}
\right)
\nonumber +
2\step^2
\left[
\mixingparam^2\left(1-\frac1{\nagent}\right)
+\frac{\mixingparam\lambda}{\nagent}
\right]
\avgnoise^2.
\end{align}
Combining the previous bound with \eqref{eq:unified-deterministic-quadratic}, we obtain
\begin{align}
\term{B}
\leq
3\step^2
\mixingparam\max\{\mixingparam,\lambda\}
\bConst{A}^2
\left(
\|\avgprojerror{t}\|_2^2
+\normdisagreement{t}
\right) +
2\step^2
\left[
\mixingparam^2\left(1-\frac1{\nagent}\right)
+\frac{\mixingparam\lambda}{\nagent}
\right]
\avgnoise^2.
\label{eq:unified-complete-quadratic-bound}
\end{align}

Since $0<\lambda\leq1$,
\begin{align}
\|\avgprojerror{t}\|_2^2+\normdisagreement{t}
\leq
\lambda^{-1}
\left(
\lambda\|\avgprojerror{t}\|_2^2
+\normdisagreement{t}
\right).
\end{align}
Hence
\begin{align}
\nonumber
3\step^2
\mixingparam\max\{\mixingparam,\lambda\}
\bConst{A}^2
\left(
\|\avgprojerror{t}\|_2^2
+\normdisagreement{t}
\right) \leq
3\step^2\mixingparam
\frac{\max\{\mixingparam,\lambda\}}{\lambda}
\bConst{A}^2
\left(
\lambda\|\avgprojerror{t}\|_2^2
+\normdisagreement{t}
\right).
\end{align}
By the second condition in
\eqref{eq:unified-lambda-conditions},
\[
\frac{\step\bConst{A}^2}{a}
\frac{\max\{\mixingparam,\lambda\}}{\lambda}
\leq\frac16,
\]
so that
\begin{align}
3\step^2
\mixingparam\max\{\mixingparam,\lambda\}
\bConst{A}^2
\left(
\|\avgprojerror{t}\|_2^2
+\normdisagreement{t}
\right)
\leq
\frac{\step\mixingparam a}{2}
\left(
\lambda\|\avgprojerror{t}\|_2^2
+\normdisagreement{t}
\right).
\label{eq:unified-direct-quadratic-absorption}
\end{align}

Substituting \eqref{eq:unified-A-bound},
\eqref{eq:unified-complete-quadratic-bound}, and
\eqref{eq:unified-direct-quadratic-absorption} into
\eqref{eq:unified-weighted-exact-expansion} yields
\begin{align}
\nonumber
\PE\left[
\normdisagreement{t+1}
+\mixingparam\lambda\|\avgprojerror{t+1}\|_2^2
\,\middle|\,\filtr^t
\right]
\nonumber
&\leq
\normdisagreement{t}
+\mixingparam\lambda\|\avgprojerror{t}\|_2^2
-\frac32\step\mixingparam a
\left(
\lambda\|\avgprojerror{t}\|_2^2
+\normdisagreement{t}
\right)
\\
\nonumber
&\quad
+\frac12\step\mixingparam a
\left(
\lambda\|\avgprojerror{t}\|_2^2
+\normdisagreement{t}
\right)
\\
&\quad+
2\step^2
\left[
\mixingparam^2\left(1-\frac1{\nagent}\right)
+\frac{\mixingparam\lambda}{\nagent}
\right]\avgnoise^2
\\
\nonumber
&=
\normdisagreement{t}
+\mixingparam\lambda\|\avgprojerror{t}\|_2^2
-\step\mixingparam a
\left(
\lambda\|\avgprojerror{t}\|_2^2
+\normdisagreement{t}
\right)
\\
&\quad+
2\step^2
\left[
\mixingparam^2\left(1-\frac1{\nagent}\right)
+\frac{\mixingparam\lambda}{\nagent}
\right]\avgnoise^2.
\label{eq:unified-pre-final-recursion}
\end{align}
Finally, since $0<\mixingparam\leq1$,
\begin{align}
\lambda\|\avgprojerror{t}\|_2^2
+\normdisagreement{t}
\geq
\mixingparam\lambda\|\avgprojerror{t}\|_2^2
+\normdisagreement{t}.
\end{align}
Applying this inequality to
\eqref{eq:unified-pre-final-recursion} gives
\begin{align}
\nonumber
\PE\left[
\normdisagreement{t+1}
+\mixingparam\lambda\|\avgprojerror{t+1}\|_2^2
\,\middle|\,\filtr^t
\right]
\leq
(1-\step\mixingparam a)
\left(
\normdisagreement{t}
+\mixingparam\lambda\|\avgprojerror{t}\|_2^2
\right)
\nonumber +
2\step^2
\left[
\mixingparam^2\left(1-\frac1{\nagent}\right)
+\frac{\mixingparam\lambda}{\nagent}
\right]\avgnoise^2.
\end{align}
Taking total expectations proves
\eqref{eq:unified-weighted-energy-recursion}.
\end{proof}
The following Lemma shows that the condition on $\lambda$ can always be met. 
\begin{lemma}
\label{lem:lambda-choice}
Assume $
0<\mixingparam\leq1,$ and $\step\leq\frac{a}{16\bConst{A}^2}.$
Then the conditions
\begin{align}
(1-\lambda)\hgty
\leq \frac{a\sqrt{\lambda}}{2},
\qquad
\frac{\step\bConst{A}^2}{a}
\frac{\max\{\mixingparam,\lambda\}}{\lambda}
\leq\frac16
\label{eq:lambda-conditions}
\end{align}
are satisfied with the following choice of $\lambda$
\begin{align}
\lambda
=
\min\left\{
1,\;
\max\left\{
4\frac{\hgty^2}{a^2},
\;
6\frac{\step\mixingparam\bConst{A}^2}{a}
\right\}
\right\}
\label{eq:explicit-lambda-choice}
\end{align}
\end{lemma}
\begin{proof}
We distinguish two cases.

First, suppose that the minimum in
\eqref{eq:explicit-lambda-choice} is attained at $1$, so that
$\lambda=1$. Then the first condition in
\eqref{eq:lambda-conditions} holds trivially since
\[
(1-\lambda)\hgty=0.
\]
Moreover, since $0<\mixingparam\leq1=\lambda$,
\[
\frac{\max\{\mixingparam,\lambda\}}{\lambda}=1.
\]
Therefore, using the step-size condition,
\[
\frac{\step\bConst{A}^2}{a}
\frac{\max\{\mixingparam,\lambda\}}{\lambda}
=
\frac{\step\bConst{A}^2}{a}
\leq\frac1{16}
\leq\frac16.
\]
Hence both conditions are satisfied.

Suppose now that the minimum in
\eqref{eq:explicit-lambda-choice} is strictly smaller than $1$. Then
\begin{align}
\lambda
=
\max\left\{
4\frac{\hgty^2}{a^2},
\;
6\frac{\step\mixingparam\bConst{A}^2}{a}
\right\}.
\end{align}
In particular,
\[
\lambda\geq4\frac{\hgty^2}{a^2},
\]
and hence
\[
\sqrt{\lambda}\geq\frac{2\hgty}{a}.
\]
It follows that
\[
\frac{a\sqrt{\lambda}}2\geq\hgty.
\]
Since $0<\lambda\leq1$,
\[
(1-\lambda)\hgty
\leq\hgty
\leq\frac{a\sqrt{\lambda}}2,
\]
which proves the first condition in
\eqref{eq:lambda-conditions}.

It remains to verify the second condition. If
$\lambda\geq\mixingparam$, then
\[
\frac{\max\{\mixingparam,\lambda\}}{\lambda}=1,
\]
and therefore
\[
\frac{\step\bConst{A}^2}{a}
\frac{\max\{\mixingparam,\lambda\}}{\lambda}
=
\frac{\step\bConst{A}^2}{a}
\leq\frac1{16}
\leq\frac16.
\]
If instead $\lambda<\mixingparam$, then
$\max\{\mixingparam,\lambda\}=\mixingparam$. Furthermore, by
\eqref{eq:explicit-lambda-choice},
\[
\lambda
\geq
6\frac{\step\mixingparam\bConst{A}^2}{a}.
\]
Consequently,
\[
\frac{\step\bConst{A}^2}{a}
\frac{\max\{\mixingparam,\lambda\}}{\lambda}
=
\frac{\step\bConst{A}^2}{a}
\frac{\mixingparam}{\lambda}
\leq
\frac16.
\]
Thus both conditions in \eqref{eq:lambda-conditions} hold in all cases.
\end{proof}
\begin{corollary}
\label{cor:unified-weighted-energy-unrolled}
Under the assumptions of \Cref{lem:unified-weighted-energy}, for every
$t\geq 0$,
\begin{align}
\nonumber
\PE\left[
\normdisagreement{t}
+\mixingparam\lambda\|\avgprojerror{t}\|_2^2
\right]
\leq
(1-\step\mixingparam a)^t
\left(
\normdisagreement{0}
+\mixingparam\lambda\|\avgprojerror{0}\|_2^2
\right) +
\frac{2\step}{a}
\left[
\mixingparam\left(1-\frac1{\nagent}\right)
+\frac{\lambda}{\nagent}
\right]\avgnoise^2.
\label{eq:unified-weighted-energy-unrolled}
\end{align}
\end{corollary}

\begin{proof}
Unrolling the recursion in \Cref{lem:unified-weighted-energy} concludes the proof.
\end{proof}
We now couple the weighted energy with the fast average recursion.

\section{A tighter Analysis of \texorpdfstring{\perlsa}{PF-LSA} through consensus/disagreement analysis}
\label{sec:tighter_analysis}

\begin{lemma}[Fast average recursion]
\label{lem:unified-fast-average}
Under the assumptions and step-size condition of
\Cref{lem:unified-weighted-energy}, for every
$t\geq 0$,
\begin{align}
\nonumber
\PE\left[
\|\avgprojerror{t+1}\|_2^2\mid\filtr^t
\right]
\leq
(1-\step a)\|\avgprojerror{t}\|_2^2
+
\left(
2\step^2\bConst{A}^2+
\frac{4\step\hgty^2}{a}
\right)\normdisagreement{t}
+\frac{2\step^2}{\nagent}\avgnoise^2.
\label{eq:quadratic-average-recursion}
\end{align}
\end{lemma}

\begin{proof}
Conditionally on $\filtr^t$, the average-error recursion gives
\begin{align}
\nonumber
\PE\left[
\|\avgprojerror{t+1}\|_2^2\mid\filtr^t
\right]
&\leq
\|\avgprojerror{t}\|_2^2
-2\step\left\langle
\avgprojerror{t},
\frac1{\nagent}\sum_{c=1}^{\nagent}
\bA[c]\projerror{c}{t}
\right\rangle
\\
&\quad+
2\step^2\bConst{A}^2
\frac1{\nagent}\sum_{c=1}^{\nagent}
\|\projerror{c}{t}\|_2^2
+\frac{2\step^2}{\nagent}\avgnoise^2.
\label{eq:unified-average-expansion}
\end{align}
The last term uses conditional centering and cross-agent independence.
Since $\sum_c\disagreement{c}{t}=0$,
\begin{align}
\frac1{\nagent}\sum_c\bA[c]\projerror{c}{t}
&=\bA\avgprojerror{t}
+\frac1{\nagent}\sum_c
(\bA[c]-\bA)\disagreement{c}{t}.
\label{eq:unified-average-split}
\end{align}
Jensen's inequality and \Cref{assum:exp_stability} imply
\begin{align}
\|\avgprojerror{t}\|_2^2
-2\step\langle\avgprojerror{t},
\bA\avgprojerror{t}\rangle
&\leq
(1-\step a)^2\|\avgprojerror{t}\|_2^2.
\end{align}
The remaining deterministic cross-term is bounded by
\begin{align}
2\step\hgty\|\avgprojerror{t}\|_2
\sqrt{\normdisagreement{t}}
&\leq
\frac{\step a}{4}\|\avgprojerror{t}\|_2^2
+\frac{4\step\hgty^2}{a}\normdisagreement{t}.
\end{align}
Finally,
$\nagent^{-1}\sum_c\|\projerror{c}{t}\|_2^2
=\|\avgprojerror{t}\|_2^2+\normdisagreement{t}$, and the step-size
condition gives
\[
(1-\step a)^2+\frac{\step a}{4}
+2\step^2\bConst{A}^2\leq1-\step a.
\]
Substitution in \eqref{eq:unified-average-expansion} proves the lemma.
\end{proof}

\begin{theorem}[Unified all-heterogeneity bound]
\label{thm:unified-all-heterogeneity}
Under the assumptions and step-size condition of
\Cref{lem:unified-weighted-energy}, let $0<\mixingparam\leq1$ and let
$0<\lambda\leq1$ satisfy
\eqref{eq:unified-lambda-conditions}. Then, for every $t\geq0$,
\begin{align}
\nonumber
\frac1{\nagent}\sum_{c=1}^{\nagent}
\PE\|\projerror{c}{t}\|_2^2
&\leq
(1-\step a)^t
\|\avgprojerror{0}\|_2^2 +
\left[
2+
\frac{8\hgty^2}{a^2}
\right]
\left(
1-\frac{\step\mixingparam a}{2}
\right)^t \times
\left(
\normdisagreement{0}
+\mixingparam\lambda
\|\avgprojerror{0}\|_2^2
\right)
\\
&+
\frac{2\step}{a\nagent}\avgnoise^2
+
\frac{2\step}{a}
\left(
2+
4\frac{\hgty^2}{a^2}
\right)
\left[
\mixingparam
\left(1-\frac1{\nagent}\right)
+\frac{\lambda}{\nagent}
\right]
\avgnoise^2.
\label{eq:unified-all-heterogeneity-first}
\end{align}
\end{theorem}

\begin{proof}
Taking total expectations in
Lemma \ref{lem:unified-fast-average} gives
\begin{align}
\PE\|\avgprojerror{t+1}\|_2^2
\leq
(1-\step a)
\PE\|\avgprojerror{t}\|_2^2
+
\step a
\left(
4\frac{\hgty^2}{a^2}
+2\frac{\step\bConst{A}^2}{a}
\right)
\PE[\normdisagreement{t}]
+
\frac{2\step^2}{\nagent}\avgnoise^2.
\label{eq:unified-average-for-coupling}
\end{align}
We next control the disagreement term using
Corollary \ref{cor:unified-weighted-energy-unrolled}. We have
\begin{align}
\PE[\normdisagreement{t}]
\leq
\PE\left[
\normdisagreement{t}
+
\mixingparam\lambda
\|\avgprojerror{t}\|_2^2
\right]  \leq
(1-\step\mixingparam a)^t
\left(
\normdisagreement{0}
+\mixingparam\lambda
\|\avgprojerror{0}\|_2^2
\right) +
\frac{2\step}{a}
\left[
\mixingparam
\left(1-\frac1{\nagent}\right)
+\frac{\lambda}{\nagent}
\right]
\avgnoise^2.
\label{eq:unified-disagreement-exact}
\end{align}
Since
\[
0<1-\step\mixingparam a
\leq
1-\frac{\step\mixingparam a}{2}<1,
\]
we can relax this bound to
\begin{align}
\PE[\normdisagreement{t}]
\leq
\left(1-\frac{\step\mixingparam a}{2}\right)^t
\left(
\normdisagreement{0}
+\mixingparam\lambda
\|\avgprojerror{0}\|_2^2
\right) +
\frac{2\step}{a}
\left[
\mixingparam
\left(1-\frac1{\nagent}\right)
+\frac{\lambda}{\nagent}
\right]
\avgnoise^2.
\label{eq:unified-disagreement-from-weighted}
\end{align}
We now unroll
\eqref{eq:unified-average-for-coupling}. Iterating the recursion from
$0$ to $t-1$ gives
\begin{align}
\nonumber
\PE\|\avgprojerror{t}\|_2^2
&\leq
(1-\step a)^t
\|\avgprojerror{0}\|_2^2 +
\step a
\left(
4\frac{\hgty^2}{a^2}
+2\frac{\step\bConst{A}^2}{a}
\right)
\sum_{j=0}^{t-1}
(1-\step a)^{t-1-j}
\PE[\normdisagreement{j}]
\\
&+
\frac{2\step^2}{\nagent}
\sum_{j=0}^{t-1}
(1-\step a)^{t-1-j}
\avgnoise^2.
\label{eq:unified-average-unrolled}
\end{align}
Substituting
\eqref{eq:unified-disagreement-from-weighted}
into \eqref{eq:unified-average-unrolled} yields
\begin{align}
\nonumber
\PE\|\avgprojerror{t}\|_2^2
&\leq
(1-\step a)^t
\|\avgprojerror{0}\|_2^2
+
\step a
\left(
4\frac{\hgty^2}{a^2}
+2\frac{\step\bConst{A}^2}{a}
\right)
\left(
\normdisagreement{0}
+\mixingparam\lambda
\|\avgprojerror{0}\|_2^2
\right)
\\
\nonumber
&\qquad\qquad\times
\sum_{j=0}^{t-1}
(1-\step a)^{t-1-j}
\left(
1-\frac{\step\mixingparam a}{2}
\right)^j
\\
\nonumber
&\quad+
2\step^2
\left(
4\frac{\hgty^2}{a^2}
+2\frac{\step\bConst{A}^2}{a}
\right)
\left[
\mixingparam
\left(1-\frac1{\nagent}\right)
+\frac{\lambda}{\nagent}
\right]
\avgnoise^2
\times
\sum_{j=0}^{t-1}
(1-\step a)^{t-1-j}
\\
&\quad+
\frac{2\step^2}{\nagent}
\sum_{j=0}^{t-1}
(1-\step a)^{t-1-j}
\avgnoise^2.
\label{eq:unified-average-after-disagreement}
\end{align}

We now bound the two geometric sums. Using Lemma \ref{lem:sum_exponents}, we have
\begin{align}
\nonumber
\sum_{j=0}^{t-1}
(1-\step a)^{t-1-j}
\left(
1-\frac{\step\mixingparam a}{2}
\right)^j
&=
\frac{
\left(1-\frac{\step\mixingparam a}{2}\right)^t
-(1-\step a)^t
}{
(1-\frac{\step\mixingparam a}{2})
-(1-\step a)
}
\\
&=
\frac{
\left(1-\frac{\step\mixingparam a}{2}\right)^t
-(1-\step a)^t
}{
\step a\left(1-\frac{\mixingparam}{2}\right)
}
\\
&\leq
\frac{
\left(1-\frac{\step\mixingparam a}{2}\right)^t
}{
\step a\left(1-\frac{\mixingparam}{2}\right)
} \leq \frac{
2\left(1-\frac{\step\mixingparam a}{2}\right)^t
}{
\step a
}  .
\label{eq:unified-two-clock-convolution}
\end{align}
For the remaining geometric sums,
\begin{align}
\sum_{j=0}^{t-1}
(1-\step a)^{t-1-j}
=
\sum_{j=0}^{t-1}
(1-\step a)^j
\leq
\frac{1}{\step a}.
\label{eq:unified-fast-geometric-sum}
\end{align}

Substituting
\eqref{eq:unified-two-clock-convolution} and
\eqref{eq:unified-fast-geometric-sum} into
\eqref{eq:unified-average-after-disagreement}, we obtain
\begin{align}
\nonumber
\PE\|\avgprojerror{t}\|_2^2
&\leq
(1-\step a)^t
\|\avgprojerror{0}\|_2^2
+
(
8\hgty^2/a^2
+4\step\bConst{A}^2/a
)
\left(
1-\frac{\step\mixingparam a}{2}
\right)^t \times
\left(
\normdisagreement{0}
+\mixingparam\lambda
\|\avgprojerror{0}\|_2^2
\right)
\\
&\quad+
\frac{2\step}{a}
\left(
4\frac{\hgty^2}{a^2}
+2\frac{\step\bConst{A}^2}{a}
\right)
\left[
\mixingparam
\left(1-\frac1{\nagent}\right)
+\frac{\lambda}{\nagent}
\right]
\avgnoise^2 +
\frac{2\step}{a\nagent}
\avgnoise^2.
\label{eq:unified-average-final}
\end{align}

Finally, the consensus--disagreement decomposition gives
\begin{align}
\frac1{\nagent}\sum_{c=1}^{\nagent}
\PE\|\projerror{c}{t}\|_2^2
=
\PE\|\avgprojerror{t}\|_2^2
+
\PE[\normdisagreement{t}].
\label{eq:unified-mse-decomposition}
\end{align}
Combining  \eqref{eq:unified-disagreement-from-weighted} with \eqref{eq:unified-average-final} gives
\begin{align}
\nonumber
\frac1{\nagent}\sum_{c=1}^{\nagent}
\PE\|\projerror{c}{t}\|_2^2
&\leq
(1-\step a)^t
\|\avgprojerror{0}\|_2^2 +
\left[
1+
\frac{8\hgty^2}{a^2}
+\frac{4\step\bConst{A}^2}{a}
\right]
\left(
1-\frac{\step\mixingparam a}{2}
\right)^t \times
\left(
\normdisagreement{0}
+\mixingparam\lambda
\|\avgprojerror{0}\|_2^2
\right)
\\
&+
\frac{2\step}{a\nagent}\avgnoise^2
+
\frac{2\step}{a}
\left(
1+
4\frac{\hgty^2}{a^2}
+2\frac{\step\bConst{A}^2}{a}
\right)
\left[
\mixingparam
\left(1-\frac1{\nagent}\right)
+\frac{\lambda}{\nagent}
\right]
\avgnoise^2.
\end{align}
Finally, using that $
\step \leq 
a/16\bConst{A}^2$ proves
\eqref{eq:unified-all-heterogeneity-first}
\end{proof}

From the previous bound, we can deduce the following two corollaries,
corresponding respectively to the high- and low-heterogeneity regimes.

\begin{corollary}[High-heterogeneity regime]
\label{cor:unified-high-heterogeneity}
Under the assumptions and step-size condition of
\Cref{thm:unified-all-heterogeneity}, assume that
$\hgty^2 \geq a^2/4$.
Then, for every $0<\mixingparam\leq1$ and every $t\geq0$,
\begin{align}
\nonumber
\frac1{\nagent}\sum_{c=1}^{\nagent}
\PE\|\projerror{c}{t}\|_2^2
&\leq
(1-\step a)^t\|\avgprojerror{0}\|_2^2
+
16\frac{\hgty^2}{a^2}
\left(
1-\frac{\step\mixingparam a}{2}
\right)^t
\left(
\normdisagreement{0}
+\mixingparam\|\avgprojerror{0}\|_2^2
\right) +
\frac{\hgty^2}{a^2}
\cdot\frac{\step}{a}
\left[
24\mixingparam
+
\frac{32}{\nagent}
\right]
\avgnoise^2.
\label{eq:unified-high-heterogeneity}
\end{align}
\end{corollary}

\begin{proof}
By Lemma \ref{lem:lambda-choice}, we may choose
\begin{align}
\lambda
=
\min\left\{
1,\;
\max\left\{
4\frac{\hgty^2}{a^2},
6\frac{\step\mixingparam\bConst{A}^2}{a}
\right\}
\right\}.
\end{align}
Since $\hgty^2\geq a^2/4$, we have $
4\hgty^2/a^2 \geq1,$
and therefore $\lambda=1$. Applying \Cref{thm:unified-all-heterogeneity} with $\lambda=1$ gives
\begin{align}
\nonumber
\frac1{\nagent}\sum_{c=1}^{\nagent}
\PE\|\projerror{c}{t}\|_2^2
&\leq
(1-\step a)^t\|\avgprojerror{0}\|_2^2 +
\left(
2+8\frac{\hgty^2}{a^2}
\right)
\left(
1-\frac{\step\mixingparam a}{2}
\right)^t
\left(
\normdisagreement{0}
+\mixingparam\|\avgprojerror{0}\|_2^2
\right)
\\
&\quad+
\frac{2\step}{a\nagent}\avgnoise^2
+
\frac{2\step}{a}
\left(
2+4\frac{\hgty^2}{a^2}
\right)
\left[
\mixingparam\left(1-\frac1{\nagent}\right)
+\frac1{\nagent}
\right]
\avgnoise^2.
\label{eq:high-before-simplification}
\end{align}
We now simplify the two coefficients. Since
$\hgty^2/a^2\geq1/4$, we have $
2+8\hgty^2/a^2
\leq
16 \hgty^2/ a^2.
\label{eq:high-transient-simplification} $
For the noise term, using
$1-1/\nagent\leq1$, we obtain
\begin{align}
\frac{2\step}{a\nagent}
+
\frac{2\step}{a}
\left(
2+4\frac{\hgty^2}{a^2}
\right)
\left[
\mixingparam\left(1-\frac1{\nagent}\right)
+\frac1{\nagent}
\right]
\leq
\frac{\step}{a}
\left[
\left(
4+8\frac{\hgty^2}{a^2}
\right)\mixingparam
+
\frac{
6+8\hgty^2/a^2
}{\nagent}
\right].
\label{eq:high-noise-intermediate}
\end{align}
Again, since $\hgty^2/a^2\geq1/4$,
\begin{align}
4+8\frac{\hgty^2}{a^2}
\leq
24\frac{\hgty^2}{a^2} \eqsp, \quad
6+8\frac{\hgty^2}{a^2}
\leq
32\frac{\hgty^2}{a^2}.
\end{align}
Consequently,
\begin{align}
\frac{2\step}{a\nagent}
+
\frac{2\step}{a}
\left(
2+4\frac{\hgty^2}{a^2}
\right)
\left[
\mixingparam\left(1-\frac1{\nagent}\right)
+\frac1{\nagent}
\right]
\leq
\frac{\hgty^2}{a^2}
\cdot\frac{\step}{a}
\left[
24\mixingparam
+\frac{32}{\nagent}
\right].
\end{align}
Combining this inequality with
\eqref{eq:high-transient-simplification} in
\eqref{eq:high-before-simplification} proves the result
\end{proof}

\begin{corollary}[Low-heterogeneity regime]
\label{cor:unified-low-heterogeneity}
Under the assumptions and step-size condition of
\Cref{thm:unified-all-heterogeneity}, assume that
$\hgty^2<a^2/4$.
Then, for every $0<\mixingparam\leq1$ and every $t\geq0$,
\begin{align}
\nonumber
\frac1{\nagent}\sum_{c=1}^{\nagent}
\PE\|\projerror{c}{t}\|_2^2
&\leq
(1-\step a)^t\|\avgprojerror{0}\|_2^2
+
4
\left(
1-\frac{\step\mixingparam a}{2}
\right)^t
\left[
\normdisagreement{0}
+
4\mixingparam\frac{\hgty^2}{a^2}
\|\avgprojerror{0}\|_2^2
\right]
\\
&\quad+
\frac{6\step}{a}
\left[
\mixingparam
+\frac5{\nagent}
\right]
\avgnoise^2
+
24\frac{\step\mixingparam^2\bConst{A}^2}{a}
\|\avgprojerror{0}\|_2^2.
\label{eq:unified-low-heterogeneity}
\end{align}
\end{corollary}

\begin{proof}
Since $\hgty^2<a^2/4$, we have $
4\hgty^2/a^2<1.$
Moreover, by the step-size condition and
$0<\mixingparam\leq1$,
\begin{align}
6\frac{\step\mixingparam\bConst{A}^2}{a}
\leq
\frac{6}{16}
=
\frac38
<1.
\end{align}
Hence the condition on $\lambda$ in Lemma \ref{lem:lambda-choice} becomes
\begin{align}
\lambda
=
\max\left\{
4\frac{\hgty^2}{a^2},
6\frac{\step\mixingparam\bConst{A}^2}{a}
\right\}.
\end{align}
In particular,
\begin{align}
\lambda
\leq
4\frac{\hgty^2}{a^2}
+
6\frac{\step\mixingparam\bConst{A}^2}{a}.
\label{eq:low-lambda-sum}
\end{align}

Since $\hgty^2/a^2<1/4$, we also have
\begin{align}
2+8\frac{\hgty^2}{a^2}
\leq4,
\qquad
2+4\frac{\hgty^2}{a^2}
\leq3.
\label{eq:low-prefactors}
\end{align}

We first control the transient term in
\Cref{thm:unified-all-heterogeneity}. Using
\eqref{eq:low-lambda-sum} and
\eqref{eq:low-prefactors},
\begin{align}
&\nonumber
\left(
2+8\frac{\hgty^2}{a^2}
\right)
\left(
1-\frac{\step\mixingparam a}{2}
\right)^t
\left(
\normdisagreement{0}
+
\mixingparam\lambda
\|\avgprojerror{0}\|_2^2
\right)
\\
& \qquad \qquad \leq
4
\left(
1-\frac{\step\mixingparam a}{2}
\right)^t
\Bigg[
\normdisagreement{0}
+
4\mixingparam\frac{\hgty^2}{a^2}
\|\avgprojerror{0}\|_2^2
+
6\frac{\step\mixingparam^2\bConst{A}^2}{a}
\|\avgprojerror{0}\|_2^2
\Bigg]
\\
& \qquad \qquad =
4
\left(
1-\frac{\step\mixingparam a}{2}
\right)^t
\left[
\normdisagreement{0}
+
4\mixingparam\frac{\hgty^2}{a^2}
\|\avgprojerror{0}\|_2^2
\right]
\nonumber +
24\frac{\step\mixingparam^2\bConst{A}^2}{a}
\left(
1-\frac{\step\mixingparam a}{2}
\right)^t
\|\avgprojerror{0}\|_2^2
\\
& \qquad \qquad\leq
4
\left(
1-\frac{\step\mixingparam a}{2}
\right)^t
\left[
\normdisagreement{0}
+
4\mixingparam\frac{\hgty^2}{a^2}
\|\avgprojerror{0}\|_2^2
\right]
+
24\frac{\step\mixingparam^2\bConst{A}^2}{a}
\|\avgprojerror{0}\|_2^2,
\label{eq:low-transient-bound}
\end{align}
where the last inequality uses
$\left(1-\step\mixingparam a/2\right)^t\leq1$.

It remains to control the noise terms. From
\eqref{eq:low-prefactors}, $\lambda<1$, and
$1-1/\nagent\leq1$,
\begin{align}
\nonumber
\frac{2\step}{a\nagent}
+
\frac{2\step}{a}
\left(
2+4\frac{\hgty^2}{a^2}
\right)
\left[
\mixingparam\left(1-\frac1{\nagent}\right)
+\frac{\lambda}{\nagent}
\right]
&\leq
\frac{2\step}{a\nagent}
+
\frac{6\step}{a}
\left[
\mixingparam
+\frac1{\nagent}
\right]
\\
&=
\frac{6\step}{a}\mixingparam
+
\frac{8\step}{a\nagent}
\\
&\leq
\frac{6\step}{a}
\left[
\mixingparam+\frac5{\nagent}
\right].
\label{eq:low-noise-bound}
\end{align}
Substituting
\eqref{eq:low-transient-bound} and
\eqref{eq:low-noise-bound} into
\Cref{thm:unified-all-heterogeneity} proves
\eqref{eq:unified-low-heterogeneity}.
\end{proof}

\begin{corollary}[Sample complexity in the low-heterogeneity regime]
\label{cor:low-heterogeneity-sample-complexity}
Under the assumptions of
\Cref{cor:unified-low-heterogeneity}, assume that 
$\hgty^2<\frac{a^2}{4}$, and set
$\mixingparam=\frac1{\nagent}.$
Let $\epsilon>0$ and choose the step size such that
\begin{align}
\step
\leq
\min\left\{
\step_\infty,\;
\frac{a}{16\bConst{A}^2},\;
\frac{a\nagent\epsilon}{144\avgnoise^2},\;
\frac{a\nagent^2\epsilon}
{96\bConst{A}^2\|\avgprojerror{0}\|_2^2}
\right\}.
\label{eq:low-heterogeneity-sample-step}
\end{align}
Then \perlsa\, satisfies
\begin{align}
\frac1{\nagent}\sum_{c=1}^{\nagent}
\PE\|\projerror{c}{T}\|_2^2
\leq\epsilon
\end{align}
provided
\begin{align}
\nonumber
T
\geq
\frac1a
\max\Bigg\{
&
\frac1{\step_\infty},
\frac{16\bConst{A}^2}{a},
\frac{144\avgnoise^2}{a\nagent\epsilon},
\frac{96\bConst{A}^2\|\avgprojerror{0}\|_2^2}
{a\nagent^2\epsilon}
\Bigg\}
\\
&\qquad\times
\max\Bigg\{
\log\left(
\frac{4\|\avgprojerror{0}\|_2^2}{\epsilon}
\right),
\;
2\nagent
\log\left(
\frac{
16\left(
\normdisagreement{0}
+
\frac{4\hgty^2}{a^2\nagent}
\|\avgprojerror{0}\|_2^2
\right)
}{\epsilon}
\right)
\Bigg\}.
\label{eq:low-heterogeneity-sample-complexity}
\end{align}
\end{corollary}

\begin{proof}
Setting $\mixingparam=1/\nagent$ in
\Cref{cor:unified-low-heterogeneity} gives
\begin{align}
\frac1{\nagent}\sum_{c=1}^{\nagent}
\PE\|\projerror{c}{t}\|_2^2
&\leq
(1-\step a)^t
\|\avgprojerror{0}\|_2^2 +
4
\left(
1-\frac{\step a}{2\nagent}
\right)^t
\left[
\normdisagreement{0}
+
\frac{4\hgty^2}{a^2\nagent}
\|\avgprojerror{0}\|_2^2
\right] +
\frac{36\step}{a\nagent}\avgnoise^2
+
\frac{24\step\bConst{A}^2}{a\nagent^2}
\|\avgprojerror{0}\|_2^2.
\label{eq:low-heterogeneity-beta-N}
\end{align}
We require each of the four terms on the right-hand side to be at
most $\epsilon/4$. First, the two non-vanishing terms are bounded by $\epsilon/4$ whenever
\begin{align}
\frac{36\step}{a\nagent}\avgnoise^2
\leq\frac{\epsilon}{4} \eqsp, \quad 
\frac{24\step\bConst{A}^2}{a\nagent^2}
\|\avgprojerror{0}\|_2^2
\leq\frac{\epsilon}{4}.
\end{align}
These two conditions are respectively equivalent to
\begin{align}
\step
\leq
\frac{a\nagent\epsilon}{144\avgnoise^2},
\qquad
\step
\leq
\frac{a\nagent^2\epsilon}
{96\bConst{A}^2\|\avgprojerror{0}\|_2^2}.
\end{align}
Combining them with the original step-size condition gives
\eqref{eq:low-heterogeneity-sample-step}.

It remains to control the two transient terms. Using
$1-x\leq e^{-x}$, we have
\begin{align}
(1-\step a)^T
\|\avgprojerror{0}\|_2^2
&\leq
e^{-\step aT}
\|\avgprojerror{0}\|_2^2.
\end{align}
Thus, the first term is at most $\epsilon/4$ provided
\begin{align}
T
\geq
\frac1{\step a}
\log\left(
\frac{4\|\avgprojerror{0}\|_2^2}{\epsilon}
\right).
\label{eq:low-complexity-fast}
\end{align}
Similarly,
\begin{align}
\nonumber
4
\left(
1-\frac{\step a}{2\nagent}
\right)^T
\left[
\normdisagreement{0}
+
\frac{4\hgty^2}{a^2\nagent}
\|\avgprojerror{0}\|_2^2
\right]
\leq
4
\exp\left(
-\frac{\step aT}{2\nagent}
\right)
\left[
\normdisagreement{0}
+
\frac{4\hgty^2}{a^2\nagent}
\|\avgprojerror{0}\|_2^2
\right].
\end{align}
Hence the second transient term is at most $\epsilon/4$ whenever
\begin{align}
T
\geq
\frac{2\nagent}{\step a}
\log\left(
\frac{
16\left(
\normdisagreement{0}
+
\frac{4\hgty^2}{a^2\nagent}
\|\avgprojerror{0}\|_2^2
\right)
}{\epsilon}
\right).
\label{eq:low-complexity-slow}
\end{align}
Combining \eqref{eq:low-heterogeneity-sample-step},
\eqref{eq:low-complexity-fast}, and
\eqref{eq:low-complexity-slow}
gives \eqref{eq:low-heterogeneity-sample-complexity} which concludes the proof.
\end{proof}

\begin{corollary}[Sample complexity in the high-heterogeneity regime]
\label{cor:high-heterogeneity-sample-complexity}
Under the assumptions of
\Cref{cor:unified-high-heterogeneity}, assume that $\hgty^2\geq\frac{a^2}{4}$, and set
$\mixingparam=\frac1{\nagent}.$
Let $\epsilon>0$ and assume that
\begin{align}
\step
\leq
\min\left\{
\step_\infty,\;
\frac{a}{16\bConst{A}^2},\;
\frac{a^3\nagent\epsilon}
{168\hgty^2\avgnoise^2}
\right\}.
\label{eq:high-heterogeneity-step-complexity}
\end{align}
Then \perlsa\, achieves
\begin{align}
\frac1{\nagent}\sum_{c=1}^{\nagent}
\PE\|\projerror{c}{T}\|_2^2
\leq\epsilon
\end{align}
provided
\begin{align}
\nonumber
T
\geq
\frac1a
\max\left\{
\frac1{\step_\infty},
\frac{16\bConst{A}^2}{a},
\frac{168\hgty^2\avgnoise^2}
{a^3\nagent\epsilon}
\right\}
\max\Bigg\{
&
\log\left(
\frac{3\|\avgprojerror{0}\|_2^2}{\epsilon}
\right),
\\
&
2\nagent
\log\left(
\frac{
48\hgty^2
\left(
\normdisagreement{0}
+\nagent^{-1}\|\avgprojerror{0}\|_2^2
\right)
}{
a^2\epsilon
}
\right)
\Bigg\}.
\label{eq:high-heterogeneity-sample-complexity}
\end{align}
\end{corollary}
\begin{proof}
Setting $\mixingparam=1/\nagent$ in
\Cref{cor:unified-high-heterogeneity} gives
\begin{align}
\frac1{\nagent}\sum_{c=1}^{\nagent}
\PE\|\projerror{c}{t}\|_2^2
\leq
(1-\step a)^t
\|\avgprojerror{0}\|_2^2 +
16\frac{\hgty^2}{a^2}
\left(
1-\frac{\step a}{2\nagent}
\right)^t
\left(
\normdisagreement{0}
+\frac1{\nagent}\|\avgprojerror{0}\|_2^2
\right)
+
\frac{56\step\hgty^2}{a^3\nagent}
\avgnoise^2.
\label{eq:high-heterogeneity-beta-N}
\end{align}
Indeed,
\begin{align}
\frac{\hgty^2}{a^2}
\cdot\frac{\step}{a}
\left[
24\mixingparam+\frac{32}{\nagent}
\right]
=
\frac{56\step\hgty^2}{a^3\nagent}.
\end{align}

To guarantee an accuracy $\epsilon$, we require each of the three
terms on the right-hand side of
\eqref{eq:high-heterogeneity-beta-N} to be at most $\epsilon/3$.

We first control the non-vanishing variance term. The condition
\begin{align}
\frac{56\step\hgty^2}{a^3\nagent}
\avgnoise^2
\leq
\frac{\epsilon}{3}
\end{align}
is equivalent to
\begin{align}
\step
\leq
\frac{a^3\nagent\epsilon}
{168\hgty^2\avgnoise^2}.
\end{align}
Combining this requirement with the original step-size condition gives
\eqref{eq:high-heterogeneity-step-complexity}.

It remains to control the two transient terms. Since
$1-x\leq e^{-x}$,
\begin{align}
(1-\step a)^T
\|\avgprojerror{0}\|_2^2
&\leq
e^{-\step aT}
\|\avgprojerror{0}\|_2^2.
\end{align}
Thus, the first transient term is at most $\epsilon/3$ whenever
\begin{align}
T
\geq
\frac1{\step a}
\log\left(
\frac{3\|\avgprojerror{0}\|_2^2}{\epsilon}
\right).
\label{eq:high-fast-complexity}
\end{align}
Similarly,
\begin{align}
\nonumber
16\frac{\hgty^2}{a^2}
\left(
1-\frac{\step a}{2\nagent}
\right)^T
\left(
\normdisagreement{0}
+\frac1{\nagent}\|\avgprojerror{0}\|_2^2
\right)
\leq
16\frac{\hgty^2}{a^2}
\exp\left(
-\frac{\step aT}{2\nagent}
\right)
\left(
\normdisagreement{0}
+\frac1{\nagent}\|\avgprojerror{0}\|_2^2
\right).
\end{align}
Therefore, this term is at most $\epsilon/3$ whenever
\begin{align}
T
\geq
\frac{2\nagent}{\step a}
\log\left(
\frac{
48\hgty^2
\left(
\normdisagreement{0}
+\nagent^{-1}\|\avgprojerror{0}\|_2^2
\right)
}{
a^2\epsilon
}
\right).
\label{eq:high-slow-complexity}
\end{align}
Combining
\eqref{eq:high-heterogeneity-step-complexity}, \eqref{eq:high-fast-complexity}, and
\eqref{eq:high-slow-complexity} concludes the proof.
\end{proof}

\section{Application to Personalized Federated TD-Learning}
We now specialize the preceding results
to personalized federated TD learning. Recall that, for each agent
$c\in[N]$, the local projected Bellman equation can be written as
\[
\bar A_c\theta_c^\star=\bar b_c,
\]
where
\[
\bar A_c
=
\mathbb E_{s\sim\mu_c,\,s'\sim P^{\pi,c}(\cdot\mid s)}
\left[
\phi(s)\{\phi(s)-\gamma\phi(s')\}^{\top}
\right],
\qquad
\bar b_c
=
\mathbb E_{s\sim\mu_c,\,a\sim\pi(\cdot\mid s)}
\left[
\phi(s)r_c(s,a)
\right].
\]
We denote
\[
\bar A=\frac1N\sum_{c=1}^N\bar A_c,
\qquad
\zeta_{\rm TD}
:=
\max_{c\in[N]}\|\bar A_c-\bar A\|,
\]
and
\[
\bar\sigma_{\varepsilon,{\rm TD}}^2
:=
\frac1N\sum_{c=1}^N\trace(\Sigma_c^\varepsilon).
\]
Under \Cref{assum:generative_model} and \Cref{assum:feature_design}, Claim \ref{prop:condition_check} gives
\[
\bar\sigma_{\varepsilon,{\rm TD}}^2
\leq
2(1+\gamma)^2
\left(
1+\frac1N\sum_{c=1}^N\|\theta_c^\star\|_2^2
\right).
\]

\begin{corollary}[Personalized federated TD: high-heterogeneity regime]
\label{cor:pf-td-high-heterogeneity}
Assume \Cref{assum:generative_model} and \Cref{assum:feature_design}. Let
\[
\eta
\in
\left(
0,
\min\left\{
\frac{1-\gamma}{4},
\frac{(1-\gamma)\nu}{32(1+\gamma)^2}
\right\}
\right),
\]
and let $0<\beta\leq1$. Assume that
$
\zeta_{\rm TD}^2
\geq (1-\gamma)^2\nu^2/16.$
Then, for every $t\geq0$,
\begin{align}
\nonumber
\frac1N\sum_{c=1}^N
\PE\|\theta_c^t-\theta_c^\star\|_2^2
&\leq
\left(
1-\frac{\eta(1-\gamma)\nu}{2}
\right)^t
\|\bar e^0\|_2^2
\\
\nonumber
&\quad+
\frac{64\zeta_{\rm TD}^2}
{(1-\gamma)^2\nu^2}
\left(
1-\frac{\eta\beta(1-\gamma)\nu}{4}
\right)^t
\left(
D^0+\beta\|\bar e^0\|_2^2
\right)
\\
&\quad+
\frac{\eta\zeta_{\rm TD}^2}
{(1-\gamma)^3\nu^3}
\left(
192\beta+\frac{256}{N}
\right)
\bar\sigma_{\varepsilon,{\rm TD}}^2.
\label{eq:pf-td-high-heterogeneity}
\end{align}
\end{corollary}

\begin{corollary}[Personalized federated TD: low-heterogeneity regime]
\label{cor:pf-td-low-heterogeneity}
Assume \Cref{assum:generative_model} and \Cref{assum:feature_design}. Let
\[
\eta
\in
\left(
0,
\min\left\{
\frac{1-\gamma}{4},
\frac{(1-\gamma)\nu}{32(1+\gamma)^2}
\right\}
\right),
\]
and let $0<\beta\leq1$. Assume that
$ \zeta_{\rm TD}^2
< (1-\gamma)^2\nu^2/16.$
Then, for every $t\geq0$,
\begin{align}
\nonumber
\frac1N\sum_{c=1}^N
\PE\|\theta_c^t-\theta_c^\star\|_2^2
&\leq
\left(
1-\frac{\eta(1-\gamma)\nu}{2}
\right)^t
\|\bar e^0\|_2^2
\\
\nonumber
&\quad+
4
\left(
1-\frac{\eta\beta(1-\gamma)\nu}{4}
\right)^t
\left[
D^0
+
\frac{16\beta\zeta_{\rm TD}^2}
{(1-\gamma)^2\nu^2}
\|\bar e^0\|_2^2
\right]
\\
\nonumber
&\quad+
\frac{12\eta}{(1-\gamma)\nu}
\left(
\beta+\frac5N
\right)
\bar\sigma_{\varepsilon,{\rm TD}}^2
+
\frac{48\eta\beta^2(1+\gamma)^2}
{(1-\gamma)\nu}
\|\bar e^0\|_2^2.
\label{eq:pf-td-low-heterogeneity}
\end{align}
\end{corollary}

\begin{corollary}[Sample complexity of personalized federated TD:
low-heterogeneity regime]
\label{cor:pf-td-low-sample-complexity}
Assume \Cref{assum:generative_model} and \Cref{assum:feature_design}, and suppose that $
\zeta_{\rm TD}^2
< (1-\gamma)^2\nu^2/16,$
and set $\beta=1/N$.
Let $\epsilon>0$ and choose the step size such that
\begin{align}
\eta
\leq
\min\left\{
\frac{1-\gamma}{4},
\frac{(1-\gamma)\nu}{32(1+\gamma)^2},
\frac{(1-\gamma)\nu N\epsilon}
{288\bar\sigma_{\varepsilon,{\rm TD}}^2},
\frac{(1-\gamma)\nu N^2\epsilon}
{192(1+\gamma)^2\|\bar e^0\|_2^2}
\right\}.
\label{eq:pf-td-low-step-complexity}
\end{align}
Then \perlsa\, satisfies
\begin{align}
\frac1N\sum_{c=1}^N
\PE\|\theta_c^T-\theta_c^\star\|_2^2
\leq\epsilon
\end{align}
provided
\begin{align}
\nonumber
T
\gtrsim
\frac{2}{(1-\gamma)\nu}
\max\Bigg\{
&
\frac{4}{1-\gamma},
\frac{32(1+\gamma)^2}{(1-\gamma)\nu},
\frac{288\bar\sigma_{\varepsilon,{\rm TD}}^2}
{(1-\gamma)\nu N\epsilon},
\frac{
192(1+\gamma)^2\|\bar e^0\|_2^2
}{
(1-\gamma)\nu N^2\epsilon
}
\Bigg\}
\\
&\qquad\times
\max\Bigg\{
\log\left(
\frac{4\|\bar e^0\|_2^2}{\epsilon}
\right),
2N
\log\left(
\frac{
16\left(
D^0+
\frac{
16\zeta_{\rm TD}^2
}{
(1-\gamma)^2\nu^2N
}
\|\bar e^0\|_2^2
\right)
}{
\epsilon
}
\right)
\Bigg\}.
\label{eq:pf-td-low-sample-complexity}
\end{align}
\end{corollary}
\begin{corollary}[Sample and communication complexity of personalized federated TD:
high-heterogeneity regime]
\label{cor:pf-td-high-sample-complexity}
Assume \Cref{assum:generative_model} and \Cref{assum:feature_design}, and suppose that
$
\zeta_{\rm TD}^2
\geq (1-\gamma)^2\nu^2/16$, and set $
\beta=1/N.$
Let $\epsilon>0$ and choose the step size such that
\begin{align}
\eta
\leq
\min\left\{
\frac{1-\gamma}{4},
\frac{(1-\gamma)\nu}{32(1+\gamma)^2},
\frac{
(1-\gamma)^3\nu^3N\epsilon
}{
1344\zeta_{\rm TD}^2
\bar\sigma_{\varepsilon,{\rm TD}}^2
}
\right\}.
\label{eq:pf-td-high-step-complexity}
\end{align}
Then \perlsa\, satisfies
\begin{align}
\frac1N\sum_{c=1}^N
\PE\|\theta_c^T-\theta_c^\star\|_2^2
\leq\epsilon
\end{align}
provided
\begin{align}
\nonumber
T
\gtrsim
\frac{2}{(1-\gamma)\nu}
\max\Bigg\{
&
\frac{4}{1-\gamma},
\frac{32(1+\gamma)^2}{(1-\gamma)\nu},
\frac{
1344\zeta_{\rm TD}^2
\bar\sigma_{\varepsilon,{\rm TD}}^2
}{
(1-\gamma)^3\nu^3N\epsilon
}
\Bigg\}
\\
&\qquad\times
\max\Bigg\{
\log\left(
\frac{3\|\bar e^0\|_2^2}{\epsilon}
\right),
2N
\log\left(
\frac{
192\zeta_{\rm TD}^2
\left(
D^0+\frac1N\|\bar e^0\|_2^2
\right)
}{
(1-\gamma)^2\nu^2\epsilon
}
\right)
\Bigg\}.
\label{eq:pf-td-high-sample-complexity}
\end{align}
\end{corollary}

\section{Technical Lemmas}
\begin{lemma}
\label{lem:sum_exponents}
Let $x,y \in \rset$ with $x \neq y$, and let $t \in \nset$. Then
\begin{align}
\sum_{k=0}^{t} x^k y^{\,t-k}=\frac{x^{t+1}-y^{t+1}}{x-y}.
\end{align}
\end{lemma}
\begin{proof}
We write
\[
(x-y)\sum_{k=0}^{t} x^k y^{\,t-k}
=
\sum_{k=0}^{t} x^{k+1}y^{\,t-k}
-
\sum_{k=0}^{t} x^k y^{\,t+1-k}.
\]
Reindexing the first sum yields
\[
\sum_{k=0}^{t} x^{k+1}y^{\,t-k}
=
\sum_{k=1}^{t+1} x^k y^{\,t+1-k}.
\]
Therefore,
\[
(x-y)\sum_{k=0}^{t} x^k y^{\,t-k}
=
\sum_{k=1}^{t+1} x^k y^{\,t+1-k}
-
\sum_{k=0}^{t} x^k y^{\,t+1-k}
=
x^{t+1}-y^{t+1}.
\]
Dividing by $x-y$ concludes the proof
\end{proof}

\begin{lemma}
\label{lem:pythagoras}
Let $(x_c)_{c=1}^\nagent$, and $(y_c)c_{c=1}^\nagent$ be $\nagent$ vectors of $\rset^{d}$. Denote $\bar{x}_\nagent = (1/\nagent) \sum_{c=1}^\nagent x_c$ and $\bar{y}_\nagent= (1/\nagent) \sum_{c=1}^\nagent y_c$. Then,
\[
\nagent \norm{\bar{x}_\nagent - \bar{y}_\nagent}^2 = \sum_{c=1}^\nagent \norm{x_c-y_c}^2
- \sum_{c=1}^{\nagent} \norm{x_c - \bar{x}_\nagent - (y_c - \bar{y}_\nagent)}^2
\]
\end{lemma}
\begin{proof}
Define $\operatorname{x}= [x_1^\top,\dots,x_\nagent^\top]^{\top}$ and
$\operatorname{y}= [y_1^\top,\dots,y_\nagent^\top]^{\top} \in \rset^{\nagent d}$.
Define by $\operatorname{P}$ the orthogonal projector on
\[
\mathcal{E}= \left\{ \operatorname{x} \in \rset^{\nagent d}: \operatorname{x}= [x^\top,\dots,x^\top]^{\top}, x \in \rset^d \right\} \eqsp.
\]
We  show that $\operatorname{P} \operatorname{x} = [\bar{x}_\nagent^\top, \dots, \bar{x}_{\nagent}^\top]^\top$. Note indeed that for any $\operatorname{z}=[z^\top,\dots,z^\top]^\top \in \mathcal{E}$, we get (with  a slight abuse of notations, $\pscal{\cdot}{\cdot}$ denotes the scalar product in $\rset^{\nagent d}$ and $\rset^d$)
\[
\pscal{\operatorname{x}-\operatorname{P} \operatorname{x}}{\operatorname{z}}=
\sum_{c=1}^\nagent \left\{ \pscal{x_c}{z} - \pscal{\bar{x}_\nagent}{z} \right\} = 0 \eqsp.
\]
The  proof follows from Pythagoras identity which shows that
\[
\norm{\operatorname{P} \operatorname{x} - \operatorname{P} \operatorname{y}}^2=
\norm{\operatorname{x}-\operatorname{y}}^2 -
\norm{(\operatorname{x}-\operatorname{P} \operatorname{x}) -(\operatorname{y}- \operatorname{P} \operatorname{y})}^2 )
\]
\eqsp.
\end{proof}

\begin{lemma}
\label{lem:bound_scalar_product}
Assume \Cref{assum:exp_stability} and \Cref{assum:noise-level-flsa}. For any $c \in [\nagent]$, and $x \in \rset^{d}$, and $\step \in (0, \step_{\infty})$, it holds that
\begin{align}
\pscal{x}{\bA[c]x} 
 \ge
  a \norm{x}^2 \eqsp.
\end{align}
\end{lemma}
\begin{proof}
For any $c\in[\nagent]$,
$u\in\mathbb R^d$, and $0<\step'<\step_\infty$,
\Cref{assum:exp_stability} gives
\begin{align}
\PE\left[
\left\|
(I-\step'\shnfuncA[c]{t+1})u
\right\|_2^2
\right]
\leq
(1-\step'a)^2\|u\|_2^2.
\end{align}
Expanding both sides and using
$\PE[\shnfuncA[c]{t+1}]=\bA[c]$ gives
\begin{align}
\nonumber
\|u\|_2^2
-2\step'\langle u,\bA[c]u\rangle
+\step'^2
\PE\left[
\|\shnfuncA[c]{t+1}u\|_2^2
\right]
\leq
\left(
1-2\step'a+\step'^2a^2
\right)\|u\|_2^2.
\end{align}
After cancelling $\|u\|_2^2$, dividing by $2\step'$, and rearranging,
\begin{align}
\langle u,\bA[c]u\rangle
\geq
a\|u\|_2^2
+
\frac{\step'}2
\left(
\PE\left[\|\shnfuncA[c]{t+1}u\|_2^2\right]
-a^2\|u\|_2^2
\right).
\end{align}
Letting $\step'\to0$ concludes the proof.
\end{proof}

\section{Additional experiments}

\begin{figure}[t]
\centering
\subfigure[$\zeta_{\text{ker}}=0.0,\zeta_{\text{rew}}=0.0$][$\zeta_{\text{ker}}=0$, $\zeta_{\text{rew}}=0$]{\includegraphics[width=0.24\linewidth]{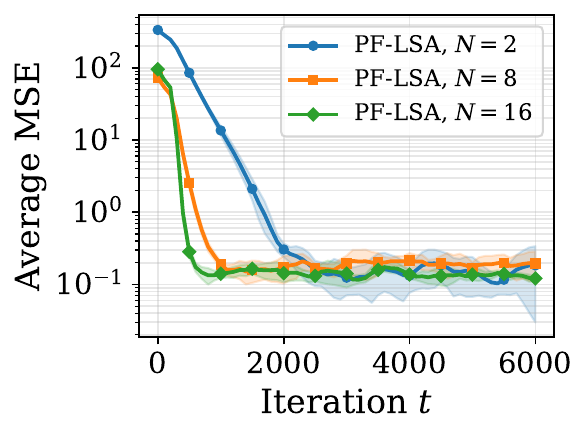}\label{fig:1app}}%
\subfigure[$\zeta_{\text{ker}}\!=\!0.01,\zeta_{\text{rew}}\!=\!0.01$][$\zeta_{\text{ker}}\!=\!0.01$, $\zeta_{\text{rew}}\!=\!0.01$]{\includegraphics[width=0.24\linewidth]{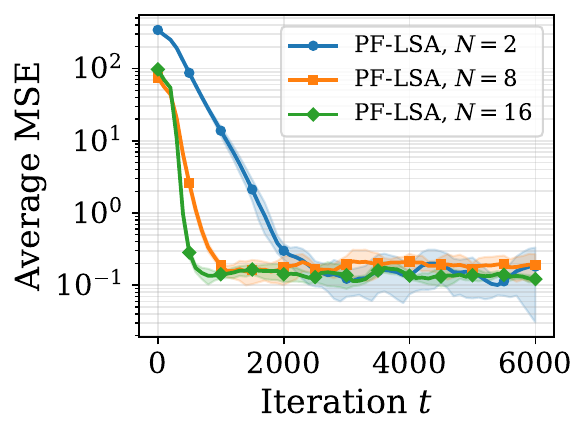}\label{fig:2app}}%
\subfigure[$\zeta_{\text{ker}}=0.5,\zeta_{\text{rew}}=0.5$][$\zeta_{\text{ker}}=0.5$, $\zeta_{\text{rew}}=0.5$]{\includegraphics[width=0.24\linewidth]{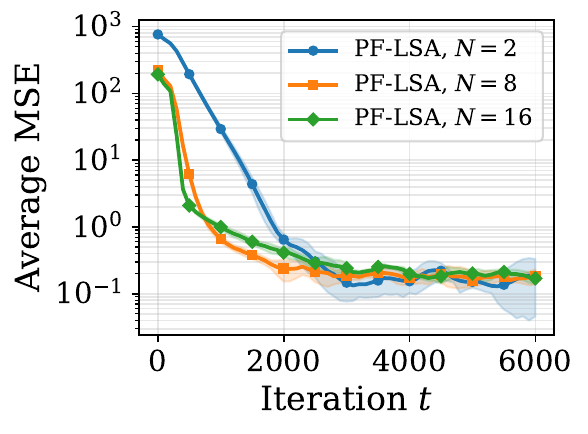}\label{fig:3app}}
\subfigure[$\zeta_{\text{ker}}=0.9,\zeta_{\text{rew}}=0.9$][$\zeta_{\text{ker}}=0.9$, $\zeta_{\text{rew}}=0.9$]{\includegraphics[width=0.24\linewidth]{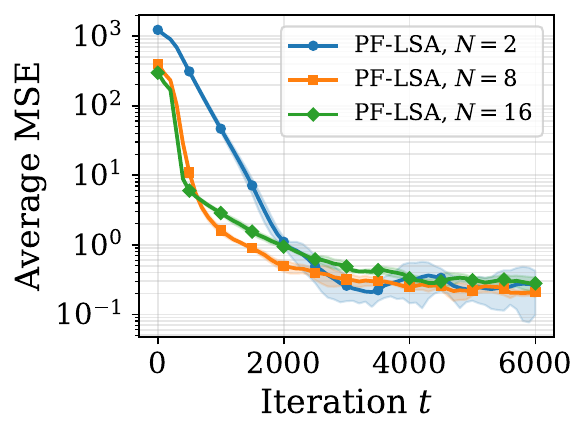}\label{fig:4app}}

\vspace{-0.1em}

\caption{Average MSE across agents,
$\frac{1}{\nagent}\sum_{c=1}^{\nagent}
\left\|\theta_c^t-\theta_c^\star\right\|_2^2$,
as a function of the number of iterations for \perlsa with
$\nagent\in\{2,8,16\}$, under increasing levels of transition-kernel
heterogeneity $\zeta_{\mathrm{ker}}$ and reward heterogeneity
$\zeta_{\mathrm{rew}}$.
The step size is set to $\step = 0.2 \times \nagent / 2$.
Each curve reports the average MSE over agents, with variability computed
over $5$ independent runs.
Consistent with our theoretical results, increasing the number of agents
yields a substantial acceleration in low-heterogeneity regimes, while the
benefit of collaboration progressively diminishes as heterogeneity increases.}
\label{fig:numerical_experiments_app}
\vspace{-0.2em}
\end{figure}

We complement the experiments of Section~\ref{sec:numerical}
by investigating directly the effect of the number of agents on the
convergence speed of \perlsa. In particular, our theoretical results predict
that, in the low-heterogeneity regime, \perlsa\, achieve a linear speedup
in the number of agents while still converging toward each agent's
personalized solution. We therefore compare the performance of \perlsa for
$\nagent \in \{2,8,16\}$ under increasing levels of transition-kernel and
reward heterogeneity.

Figure~\ref{fig:numerical_experiments_app} shows the average MSE across
agents as a function of the number of iterations. We use $\step = 0.1 \times \nagent$ and set $\mixingparam = 1/\nagent$, in accordance with our theoretical results.

\paragraph{Increasing the number of agents accelerates convergence under low heterogeneity.}
In the homogeneous setting
(Figure~\ref{fig:numerical_experiments_app}(a)), increasing the number of
agents leads to a substantial acceleration of convergence. In particular,
$\nagent=16$ reaches the low-error regime considerably earlier than
$\nagent=8$, which itself converges faster than $\nagent=2$. The same
behavior remains clearly visible under very mild heterogeneity
(Figure~\ref{fig:numerical_experiments_app}(b)). This is consistent with
our theoretical prediction that, when the agents' local problems are
sufficiently similar, averaging information across agents reduces the
effective stochastic variance and allows \perlsa\, to benefit from a speedup
that scales with the number of agents.

\paragraph{\perlsa preserves personalization across all heterogeneity levels.}

As the heterogeneity increases, the separation between the curves
corresponding to different values of $\nagent$ becomes progressively
smaller. For moderate heterogeneity
(Figure~\ref{fig:numerical_experiments_app}(c)), using more agents still
improves the transient convergence behavior, but the gain is significantly
weaker than in the nearly homogeneous cases. Under strong heterogeneity
(Figure~\ref{fig:numerical_experiments_app}(d)), the three curves become
much closer, and increasing the number of agents provides only a limited
acceleration. However, in all cases, the method successfully converges to each agent's personalized solution.

Overall, Figure~\ref{fig:numerical_experiments_app} provides empirical evidence for the best-of-both-worlds behavior predicted by our theory: \perlsa benefits increasingly from additional agents when collaboration is useful, while converging to each agent’s personalized solution across all heterogeneity levels.

\end{document}